\documentclass{article}
\usepackage{graphicx}
\usepackage[T1]{fontenc}
\usepackage[utf8]{inputenc}
\usepackage[english]{babel}
\usepackage{amsmath,amsthm,amssymb,amsfonts}
\usepackage{bm}

\usepackage{booktabs} 
\usepackage{threeparttable}
\usepackage{makecell}
\usepackage{multirow}

\usepackage{lmodern}
\usepackage{subcaption}
\usepackage{geometry}
 \usepackage[nice]{nicefrac}
\usepackage{mathabx}
 
\usepackage[colorinlistoftodos,textwidth=2.1cm]{todonotes}

\usepackage{tcolorbox}
\usepackage{pifont}
\definecolor{mydarkgreen}{RGB}{39,130,67}
\definecolor{mydarkred}{RGB}{192,25,25}
\definecolor{mydarkblue}{RGB}{0,0,140}

 \usepackage{empheq}
 \usepackage{color}
\definecolor{darkgreen}{rgb}{0.00,0.5,0.00}
\usepackage[pagebackref=true,backref=true]{hyperref}
\hypersetup{
	colorlinks=true,        
	linkcolor=blue,         
	citecolor=mydarkblue,         
	urlcolor=blue           
}

\usepackage{cite}
\usepackage{natbib}
\usepackage{tcolorbox}
\usepackage{framed}
\usepackage{algorithm}
\usepackage{algpseudocode}
\usepackage{enumerate}
\usepackage{cleveref}
\newtheorem{theorem}{Theorem}
\newtheorem{assumption}{Assumption}
\newtheorem{definition}{Definition}
\newtheorem{lemma}{Lemma}
\newtheorem{proposition}{Proposition}

\crefname{assumption}{assumption}{assumptions}

\theoremstyle{definition}

\newtheorem{example}{Example}

\usepackage{sidecap}

\let\lll\ll
\renewcommand{\ll}{\mathbf{l}}

\newcommand{\cL}{\mathcal{L}}
\newcommand{\cS}{\mathcal{S}}
\newcommand{\cO}{\mathcal O}

\newcommand{\R}{\mathbb R}

\newcommand{\eqdef}{\stackrel{\text{def}}{=}}

\DeclareMathOperator{\argmin}{argmin}

\def\<#1,#2>{\langle #1,#2\rangle}

\usepackage{dsfont}

\newcommand{\SVD}{{\rm SVD}}

\renewcommand{\leq}{\leqslant}
\renewcommand{\geq}{\geqslant}

\def\<{\langle}
\def\>{\rangle}
\def\|{\Vert}

\def\eps{\varepsilon}

\newcommand{\rank}{{\rm rank}}
\newcommand{\dist}{{\rm dist}}

\newcommand{\esp}[1]{\mathbb{E}\left[#1\right]}

\newcommand{\NRM}[1]{{{\left\| #1\right\|}}} 
\newcommand{\set}[1]{{{\left\{ #1\right\}}}} 
\newcommand{\proba}[1]{\mathbb{P}\left(#1\right)}

\newcommand{\gau}[1]{\mathcal{N}(0, #1)}

\providecommand{\myparagraph}[1]{\paragraph{#1}}

\renewcommand{\P}{\mathbb{P}}
\newcommand{\cR}{\mathcal{R}}

\newcommand{\cB}{\mathcal{B}}
\newcommand{\cN}{\mathcal{N}}
\newcommand{\cX}{\mathcal{X}}
\newcommand{\cW}{\mathcal{W}}

\newcommand{\cU}{\mathcal{U}}

\newcommand{\cF}{\mathcal{F}}

\newcommand{\N}{\mathbb{N}}
\newcommand{\cD}{\mathcal{D}}

\newcommand{\op}{{\rm op}}

\usepackage{amsfonts}
\usepackage{amsmath}
\usepackage{amssymb}
\usepackage{mathrsfs}
\usepackage{amscd}
\usepackage{caption}
\usepackage{indentfirst}

\DeclareMathOperator{\Tr}{Tr}

\title{Meta-learning of shared linear representations beyond well-specified linear regression
}

\date{}
\author{
Mathieu Even$^{1}$
\qquad \qquad
Laurent Massoulié$^{2}$
\\
\\
$^{1}$PreMeDICaL Inria-Inserm Project Team, University of Montpellier, France
\\
$^{2}$ARGO Project Team, Inria-ENS Paris, PSL Research University, France
}
\begin{document}

\maketitle

\begin{abstract}
Motivated by multi-task and meta-learning approaches, we consider the problem of learning structure shared by tasks or users, such as shared low-rank representations or clustered structures.
While all previous works focus on well-specified linear regression, we consider more general convex objectives, where the structural low-rank and cluster assumptions are expressed on the optima of each function. We show that under mild assumptions such as \textit{Hessian concentration} and \textit{noise concentration at the optimum}, rank and clustered regularized estimators recover such structure, provided the number of samples per task and the number of tasks are large enough.
    We then study the problem of recovering the subspace in which all the solutions lie, in the setting where there is only a single sample per task: we show that in that case, the rank-constrained estimator can recover the subspace, but that the number of tasks needs to scale exponentially large with the dimension of the subspace.
Finally, we provide a polynomial-time algorithm via nuclear norm constraints for learning a shared linear representation in the context of convex learning objectives.
\end{abstract}

\section{Introduction}

A central task in modern machine learning is the training of a common model from local data sets held by individual agents \citep{mcmahan2017fl}. A typical application is when (an eventually large amount of) users\footnote{In this work, the terms \textit{agent} and \textit{user} are used in an interchangeable way, as well as the term \textit{task} that refers to each agent/user objective.} (\emph{e.g.}~mobile phones, hospitals) want to make predictions (\emph{e.g.}~next-word prediction, treatment prescriptions), but each has access to very few data samples, hence the need for collaboration.
As highlighted by many recent works \citep[e.g.]{hanzely_lower_2020,mansour_three_2020}, while training a global model yields better statistical efficiency on the combined datasets of all agents by increasing the number of samples linearly in the number of agents, this approach can suffer from a dramatically poor generalization error on local datasets. 
A solution to this generalization issue is the training of \emph{personalized} models, a midway between a single model shared by all agents and models trained locally without any coordination.
An ideal approach would take the best of both worlds: increased statistical efficiency for each user by using more samples (from the pool of users), while keeping local generalization errors low.

The problem of training machine learning models personalized for each agent is an instance of \textit{multi-task learning}, each task being the local objective of each user. It is however impossible without any assumption on the structure of the tasks or similarity between agents to ensure that agents can gain any statistical efficiency from collaboration \citep[Theorem 1]{even2022on}.
There thus have been different set of assumptions, backed by empirical evidence, to design algorithms and strategies that guarantee some kind of collaboration speedup in terms of statistical guarantees.
The first set of assumptions are \textit{similarity assumptions}: each agent is assumed to be ‘‘close'' to a certain number other agents. This can be formalized as a \textit{cluster} assumption \citep{jacob2008convexcluster,liu2017hierarchical,ghosh2020clusteredFL} on the users or tasks (they form groups, and within each group/cluster the objective is the same), or as a \textit{discrepancy-based} assumption (a metric between tasks quantitatively gives how close they are and thus how an agent can benefit from collaboration with another agent) \citep{mansour_three_2020,even2022on,ding2022collaborative,deng2023distributed}. Clusters or discrepancies are usually assumed to be unknown and based on their knowledge tasks become simpler.

Assumptions can also be leveraged on the tasks themselves: they all share some common structure. The most common structural assumption that is made is the \textit{shared representation} assumption: local datasets might be heterogeneous and in very large dimension, leading to arbitrarily different tasks, yet high-dimensional data may share some representation in a much smaller space \citep{bengio2013representation}.
Then, low-dimensional representations can be translated into some structure on local objectives themselves. In the case of learning deep models for instance, \citet{pmlr-v139-collins21a} explain that tasks may share most of the layers, while personalization can be expressed only in the last ones that are user-specific.
An instance of representation assumption that is amenable for statistical analysis purposes is the \textit{shared linear representation} assumption for linear quadratic regression: for a given task $i$, data is generated according to some noisy linear relation $y_{ij}=\langle w_i^\star,x_{ij}\rangle +\eps_{ij}$ for data points $(x_{ij},y_{ij})$, $j\in[m]$ and some $w_i^\star$ to recover. A common assumption is thus to assume that all \textit{local heads} $w_i^\star\in\R^d$ for all agents lie in a small subspace of $\R^d$ of dimension $r\lll d$.

In both cases -- user assumptions or structural assumptions --, there exists a structure to be learnt that makes the problem at hand much simpler in the sense that it requires a reduced amount of samples per task to solve. This is typical of \textit{meta-learning} or \textit{learning-to-learn} approaches in deep-learning, where the process of learning itself is being improved by multiple learning episodes (that consist of the various users' task in our case) \citep{hospedales2022metalearning}. Whether a cluster structure, dissimilarity between clients, or a shared representation is known before solving each client task, the whole procedure becomes much easier.

In this work, we study the problem of jointly learning $n$ models for $n$ agents, based on the local datasets of these users, while making the kind of assumption described above. More formally, each agent $i$ indexed by $i\in[n]$ has a dataset of size $m$, $\set{\xi_{ij}}_{j\in[m]}$ with $\xi_{ij}\sim \cD_i$ and $\cD_i$ its local distribution, and wishes to minimize its \textit{generalization error}, defined as for some loss $F_i:\R^d\times \Xi\to \R$:
\begin{equation*}
	f_i(w)=\esp{F_i(w,\xi_i)}\,,\quad w\in\R^d\,,\quad \xi_i\sim\cD_i\,,
\end{equation*}
In order to capture the minimization of each local function and thus \textit{personalization}, the objective we want to minimize then writes as:
\begin{equation}\label{eq:test_loss}
	f(W)=\frac{1}{n}\sum_{i=1}^n f_i(w_i)\,,
\end{equation}
for $ W=(w_1|\ldots|w_n)\in\R^{d\times n}$.
Based on the $m$ samples per agent we have, we denote as $\cL$ the empirical risk, that writes:
\begin{equation}\label{eq:loss_emp}
	\cL(W)=\frac{1}{nm}\sum_{(i,j)\in[n]\times[m]} F_i(w_i,\xi_{ij})\,,
\end{equation}
for $ W=(w_1|\ldots|w_n)\in\R^{d\times n}$.
By minimizing this empirical risk, we wish to obtain good minimizers for the true test loss (\Cref{eq:test_loss}), under adequate assumptions.
For our problem, the structural assumptions described above then read as follows.
\begin{assumption}[Low rank representation]\label{hyp:low_rank}
	There exist $\set{w_i^\star}_{i\in[n]}$ such that for all $i\in[n]$, $w_i^\star$ minimizes $f_i$ and such that the rank of the matrix $W^\star=(w_1^\star|\ldots|w_n^\star)\in\R^{d\times n}$ is at most $r$.
    Equivalently, there exist an orthonormal matrix $U_\star\in\R^{d\times r}$ and $V^\star=(v_1^\star,\ldots,v_n^\star)\in\R^{r\times n}$ such that $w_i^\star=U_\star v_i^\star$ for all $i\in[n]$.
\end{assumption}
\begin{assumption}[Clustered clients]\label{hyp:clusters}
There exist $\set{w_i^\star}_{i\in[n]}$, $\set{c_s^\star}_{s\in[r]}$ and $\tau^\star:[n]\to[r]$ such that for all $i$, $w_i^\star$ minimizes $f_i$ and  $w_i^\star=c_{\tau^\star(i)}^\star$.
\end{assumption}
Note that \Cref{hyp:clusters} is in fact a subcase of \Cref{hyp:low_rank}.
The purpose of our paper is then to study natural minimizers under these assumptions, for convex objectives $f_i$; we provide typical examples of such tasks.

\begin{example}\label{remark:ex}
    The following examples fit in our framework.
    \begin{enumerate}
        \item \textit{Well-specified linear regression with the quadratic loss}: $\xi_{ij}=(x_{ij},y_{ij})$ where $x_{ij}$ are \textit{i.i.d.} random variable and $y_{ij}=\langle w_i^\star,x_{ij}\rangle + z_{ij}$ for independent \textit{label noise }$z_{ij}\sim\cN(0,\eps^2)$, and $F_i(w,(x,y))=\frac{1}{2}(w^\top x-y)^2$.
        \item \textit{Generalized linear models} (GLMs): $F_i(w,\xi_{ij}=(x_{ij},y_{ij})=\ell_i(\langle w,x_{ij} \rangle ,y_{ij})$, for some convex loss $\ell_i$. Classification with the logistic loss and meta learning of linearized physical systems \citep{blanke:hal-04513216} are particular instances of GLMs.
    \end{enumerate}
\end{example}


\subsection{Related works}

\noindent \myparagraph{Multi-task learning.}
Multi-task learning (MTL) is a general approach to improve generalization guarantees of single tasks knowledge of other tasks as an inductive bias \citep{Caruana1997}. On the theory side, some structural properties need to be made in order to show the advantage of MTL, such as a small variance between tasks \citep{cesa-bianchi2022multitask,pmlr-v97-denevi19a}, a shared sparse support for multi-task feature recovery \citep{argyriou2006feature,lounici2009sparsity}, and finally a low rank structure that takes the form of a lower dimensional linear representation \citep{pmlr-v178-boursier22a}.
Several convex penalties exist to learn structures shared between tasks \citep{jacob2008convexcluster,pmlr-v37-ciliberto15}, without provable guarantees.

\noindent \myparagraph{Learning clusters of clients.}
Clustering historically refers to partitioning data points into a few clusters, a task related to multi-class classification. It provides a theory-friendly framework to study and more importantly develop similarity based algorithms when in multi-task, collaborative, or federated learning, the agents are assumed to forms clusters of similar users \citep{clustered_FL}. Clustered agents can be enforced by regularized objectives \citep{mansour_three_2020,jacob2008convexcluster,zhou2011clustered}, iterative clustering algorithms \citep{ghosh2020clusteredFL}. Providing generalization guarantees and a collaborative speedup under a cluster assumption for our problem is however a challenging problem. \citet{ghosh2020clusteredFL} proves that provided a sufficiently good initialization, the clusters can be learnt with a speedup proportional to clusters' population, while \citet{even2022on} build a similarity graph and obtains similar results without any assumption on the initialization, but with degraded dimension dependency.

\noindent \myparagraph{Meta-Learning of linear representations.}
Motivated by empirical evidence in deep-learning tasks \citep{bengio2013representation,pmlr-v139-collins21a}, linear representations are a simple yet insightful proxy for theoretical works to model simple structures.
We describe in this paragraph theoretical advances in estimating such linear representations: it is to be noted that most previous works focus on the \textit{quadratic linear regression} setting, much simpler than ours, that fits a general convex framework.

\citet{pmlr-v139-tripuraneni21a} shows it is impossible to estimate a representation of dimension $r\lll d$, if the total number of samples $nm$ satisfies $nm=\cO(rd)$.
\citet{rohde2011,du2021fewshot} study the lowrank estimator as we define it, and show that the representation is learned if the total number of samples satisfies $nm\gg rd$ with $m\gg d$ samples per task: such a large number of samples per task is prohibitive, since under such assumptions collaboration is no longer needed.
\citet{maurer2016benefits} considers general convex objectives like us, but under Lipschitz (bounded gradients) assumptions, while \citet{maurer2006boundsMTL} considers general Hilbert-Schmidt operators.

Convex relaxations via nuclear norm regularization have then been studied \citep{pmlr-v178-boursier22a,rohde2011,pmlr-v30-Pontil13} and provably learn the representations provided that $n\gg d$ and $m\gg \sqrt{rd}$.
Finally, a long line of work relaxes the rank constrain by factorizing the matrix $W$ as $W=AB^\top$ where $A\in\R^{d\times r}$ and $B\in \R^{n\times r}$, an approach known as \textit{Burer-Monteiro factorization} \citep{Burer2003,Burer2004}. 
\citet{pmlr-v139-tripuraneni21a} showed in the quadratic regression setting with random inputs, that the local minima induced by Burer-Monteiro factorization all led to estimating the linear representation, provided that $nm\gg r^4d$ with $m\gg r^4$.
Then, algorithms based on alterative minimizations \citep{pmlr-v139-collins21a,thekumparampil2021statistically} require $nm\gg r^2 d$ and $m\gg r^2$, together with an initialization with non-trivial alignment with the ground truth representation.
Finally, algorithms specific to meta-learning such as MAML or ANIL \citep{pmlr-v162-collins22aMAML,yuksel2023firstorder} or even FedAvg \citep{collins2022fedavg} can learn the representation if the model is parameterized using a Burer-Monteiro factorization if the number of samples is large enough, even if the rank is mispecified \citep{yuksel2023firstorder}.

\citet{pmlr-v139-tripuraneni21a,duchi2022subspace} showed that a Method-of-Moment (MoM) estimator can recover the subspace even if $m=1$, as long as $n\gg r^2d$.
Very recently and concurrently to this work, \citet{niu2024collaborativelearningsharedlinear} showed that a clever modification of the MoM estimator leads to optimal statistical rates for learning shared linear representations in planted noisy linear regression.
MoM estimators are however specifically tailored for the planted linear regression problem.

All the aforementioned works study the linear quadratic regression setting with random inputs. Going beyond this simplified setting to learn linear representations is thus a question of interest, and \citet{collins2023provable} does so, by studying 2-layered ReLU networks in a multi-task setting. Our work also goes in this direction: we relax the assumptions to the general convex setting.
Such relaxations have also been considered in the compressed sensing literature, such as \citet{agarwal2012fast} and related works, that studied Iterative Thresholding algorithms under general restricted strong convexity and smoothness assumptions.
The broad literature on online meta learning also considers general convex objectives. However, as described below, these works cannot be compared to ours.

\noindent \myparagraph{Online Meta learning.}
There has been a lot of focus recently on \textit{online meta learning}, where tasks and sample per tasks are accessed iteratively one after the other \citep{khodak2019provable,pmlr-v98-bullins19a,pmlr-v80-lee18a,khodak2019adaptiveBDML,denevi2019onlinewithinonline,pmlr-v97-denevi19a}.
Techniques from online convex optimization are used to derive regret bounds for gradient methods.
If the optimal models for each task (in the regret sense) are close to each other, then online meta-learning techniques learn good initializations that makes learning easier for new tasks.
The regularity assumptions made in such work are the same as in convex online optimization: convex objectives and smoothness or Lipchitzness.
Such results cannot be compared to ours, since online meta learning does not learn representations, but rather learns an initialization close to all true optima, therefore replacing the dependency on the initial distance to true optima in online learning for each task by a smaller quantity \citep{khodak2019adaptiveBDML}.
In our framework, tasks' optima may however all be far from each other.

Overall, the works that have the same objective as ours (that is, establishing sample complexity rates for estimating a shared low rank representation) are either specifically tailored for linear regression and cannot be extended beyond this framework \citep{pmlr-v139-tripuraneni21a,duchi2022subspace,niu2024collaborativelearningsharedlinear}, assume bounded gradients \citep{maurer2016benefits}, or establish rates only for linear quadratic regression (\Cref{remark:ex}.1) \citep[etc]{yuksel2023firstorder,pmlr-v178-boursier22a,pmlr-v139-collins21a}.

\subsection{Our contributions and outline of the paper}

We first introduce our setting and assumptions in \Cref{sec:hyp}: the mild noise assumptions (concentration of the noise \textit{at the optimum}) and the pointwise concentration of Hessians we introduce are a core contribution of our paper.
We study rank-constrained and clustered-contrained estimators in \Cref{sec:rank} to upper bound the sample complexity of minimizing each function and learning the global structure, with applications to \textit{few-shot} learning on a new task in \Cref{sec:fewshot_meta}.
We then study the recovery of the subspace in the degenerate case where $m=1$ (only one sample per agent) via rank-constraint, and show that it is still feasible, but requires a larger number of tasks (that scales exponentially with $r$).
Finally, we relax the rank constraint into a nuclear norm constraint, and provide statistical guarantees for the obtained estimator, that can be efficiently computed.

\section{Setting, assumptions and notations}

\noindent \myparagraph{Notations.}
For $w\in\R^d$, $\NRM{w}^2=\sum_{k=1}^dw_k^2$ is its squared Euclidean norm and $\NRM{w}_1=\sum_{k=1}^d|w_k|$ is $\ell^1-$norm. For $W\in\R^{d\times n}$ with singular values $\lambda_1,\ldots,\lambda_p$, $\NRM{W}_F^2=\sum_{(i,k)\in[n]\times [d]}W_{ki}^2=\sum_{k=1}^p \lambda_k^2$ is its squared Frobenius norm, $\NRM{W}_\op=\max_{k\in[p]}|\lambda_k|$ is its operator norm, and $\NRM{W}_*=\sum_{k=1}^p|\lambda_k|$ is its nuclear norm.
For some differentiable function $g:\cX\to\R$ on $\cX\subset\R^d$, we define its \textit{Bregman divergence} $D_g:\cX\times \cX\to \R$ as: for all $w,w'\in\cX$, $D_g(w,w')=g(w)-g(w')-\langle \nabla g(w'),w-w'\rangle $.
$I_d$ is the identity matrix in $\R^d$.
For symetric matrices $A,B$, $A\preceq B$ means that all the eigenvalues of $B-A$ are non-negative. 
$\cB_d$ is the (closed) unit ball of $\R^d$, $\cS^{d-1}$ is the unit sphere.
The \textit{principal angle distance} between two subspaces or two orthogonal matrices is defined in \Cref{app:angle}.

\subsection{Main assumptions}\label{sec:hyp}

We here describe the assumptions we make on local tasks $f_i$: note that we go well beyond the classical linear regression with quadratic loss setting, where local functions are assumed to be of the $f_i(w)=\sum_j (\langle w,x_{ij}\rangle -y_{ij})^2$, for noisy labels $y_{ij}=\langle x_{ij},w_i^\star\rangle +\eps_{ij}$ and random inputs and noise $x_{ij},\eps_{ij}$.

\noindent \myparagraph{Assumptions on local functions $f_i$ and $F_i$.}
Besides using \Cref{hyp:low_rank,hyp:clusters} alternatively, that implicitly assume that each $f_i$ is lower-bounded and minimized at some $w_i^\star$, we will make use of the following regularity assumption on the local objectives $f_i$ and losses $F_i$ throughout the paper.

For all $i\in[n]$, $\NRM{w_i^\star}\leq B$, for some $B>0$.
Each $f_i$ is assumed to be $\mu$-\textbf{strongly convex} and $L$-\textbf{smooth} (with eventually $\mu\geq0$, which amounts to convexity) \citep{convexoptim_bubeck}. 
We also assume that each $F_i$ is convex in its first argument, and respectively that $\nabla^2 F_i$ and $\nabla F_i$ are $H-$ and $L'-$\textbf{Lipschitz} in their first arguments.

\noindent \myparagraph{Noise at the optimum.}
We now need to assume that noise at the optimum is not too large, which we quantify through a noise constant $\sigma_\star^2$ as follows.
\begin{assumption}\label{hyp:noise}
    For all $i\in[n]$, the random variable $\nabla_w F_i(w_i^\star,\xi_i)$ for $\xi_i\sim\cD_i$ is a $\sigma^2_\star$-multivariate subexponential random variable, in the sense that for all $w\in\cS^{d-1}$, $\langle \nabla_w F_i(w_i^\star,\xi_i), w\rangle$ is $\sigma_\star^2B$ subexponential:
\begin{equation*}
    \forall t\geq 0\,,\quad \proba{\langle \nabla_w F_i(w_i^\star,\xi_i), w\rangle \geq t}\leq \exp\left(-\frac{t}{\sigma_\star^2 B}\right)\,.
\end{equation*}
\end{assumption}
This assumption essentially describes the noise at the optimum. In the case of linear quadratic regressions \textit{i.e.} if $\xi_{ij}=(x_{ij},y_{ij})$ where $x_{ij}\sim\cN(0,I_d)$ and $y_{ij}=\langle w_i^\star,x_{ij}\rangle + z_{ij}$ for independent \textit{label noise }$z_{ij}\sim\cN(0,\eps^2)$, the gradient at the optimum reads $\nabla_w F_i(w_i^\star,\xi_{ij})=z_{ij}x_{ij}$, and thus \Cref{hyp:noise} holds for $\sigma^2_\star=\eps$.
For GLMs, we have $\nabla_w F_i(w_i^\star,\xi_{ij})=\ell'_i(\langle w_i^\star, x_{ij}\rangle, y_{ij})x_{ij}$, so that the same holds.

\noindent \myparagraph{Hessian concentration.}
Finally, we will assume that the Hessian of local functions concentrate fast enough as the number of samples increase, in the following way. 
\begin{assumption}\label{hyp:hessian}
    For all $i\in[n]$, for all $w,w_1,w_2\in\cS^{d-1}$, the random variable $\langle w_1, \nabla^2 F_i(w,\xi_i) w_2\rangle$ is a $\sigma^2$-subexponential random variable:
\begin{equation*}
    \forall t\geq 0\,,\quad \proba{\langle w_1, \nabla^2 F_i(w,\xi_i) w_2\rangle\geq t}\leq \exp\left(-\frac{t}{\sigma^2}\right)\,.
\end{equation*}
\end{assumption}
Such an assumption simply quantifies concentration of Hessians around their means, and will be used to control quantities such as $\frac{1}{nm}\sum_{ij}\langle w_i', \nabla^2 F_i(w_i,\xi_{ij}) w_i''\rangle- \frac{1}{n}\sum_i \langle w_i', \nabla^2 f_i(w_i) w_i''\rangle$.
This assumption holds for linear quadratic regression and GLMs, for which respectively $\nabla^2 F_i(w_i,(x_{ij},y_{ij})=x_{ij}x_{ij}^\top$ and $\nabla^2 F_i(w_i,(x_{ij},y_{ij})=\ell_i''(\langle x_{ij},w_i\rangle ,y_{ij})x_{ij}x_{ij}^\top$, provided that $\ell''$ is bounded.

\begin{example}
    Our assumptions hold in the following cases.
    \begin{enumerate}
        \item In the case of linear quadratic regressions \textit{i.e.} if $\xi_{ij}=(x_{ij},y_{ij})$ where $x_{ij}\sim\cN(0,I_d)$ and $y_{ij}=\langle w_i^\star,x_{ij}\rangle + z_{ij}$ for independent \textit{label noise }$z_{ij}\sim\cN(0,\eps^2)$, the gradient at the optimum reads $\nabla_w F_i(w_i^\star,\xi_{ij})=z_{ij}x_{ij}$, and thus \Cref{hyp:noise} holds for $\sigma^2_\star=\eps$.
        \Cref{hyp:hessian} holds for $\sigma^2=2$ since $\nabla^2 F_i(w_i,(x_{ij},y_{ij})=\ell_i''(\langle x_{ij},w_i\rangle ,y_{ij})x_{ij}x_{ij}^\top$.
        \item For GLMs and if inputs are random Gaussians as above, we have $\nabla_w F_i(w_i^\star,\xi_{ij})=\ell'_i(\langle w_i^\star, x_{ij}\rangle, y_{ij})x_{ij}$ and $\nabla^2 F_i(w_i,(x_{ij},y_{ij})=\ell_i''(\langle x_{ij},w_i\rangle ,y_{ij})x_{ij}x_{ij}^\top$, so that \Cref{hyp:noise} and \Cref{hyp:hessian} hold for $\sigma^2_\star=L_\ell B\eps$ and $\sigma^2=2H_\ell $, if the first and second order derivatives $\ell_i'$ and $\ell_i''$ with respect to their first argument are assumed to be respectively $L_\ell-$ and $H_\ell-$Lipschitz.
    \end{enumerate}
\end{example}

Previous works that aim at deriving sample complexity rates for learning a shared representation all consider only linear quadratic regression \citep{pmlr-v139-tripuraneni21a,thekumparampil2021statistically,niu2024collaborativelearningsharedlinear} or consider Lipschitz objective functions \citep{maurer2016benefits}.
More general convex objectives functions are considered in the online meta learning framework; however, this line of work learns good initializations for new tasks rather than a shared representation.

Our assumptions are thus much milder than previous related works, and capture exactly the noise at the optimum and the way the functions concentrate, therefore leading to rates for learning shared representations.

\subsection{Baselines without collaboration, in idealized cases, and structured regularization}\label{sec:no_collab}

The simple baseline we should compare ourselves to -- and that we desire to outperform --, is the \emph{no-collaboration} estimator $\hat W_{bad}=(\hat w_1,\ldots,\hat w_n)$ where $\hat w_i$ is defined as a minimizer of agent $i$'s empirical loss, without any collaboration.
This yields a performance of 
$f(\hat W_{bad})-f(W^\star)=\cO\left(\sqrt{\frac{d}{m}}\right)\,,$ or $\frac{1}{n}\sum_{i=1}^n\NRM{\hat w_{bad,i} -w_i^\star}_F^2=\cO\left(\frac{d}{m}\right)\,,$
if $f_i$ are respectively assumed to be convex or strongly convex.
Such an estimator that does not use collaboration thus requires $m=\Omega(d)$ samples per agent \citep[\textit{e.g.}, for linear regression]{wainwright_2019}.
Our ultimate goal is thus to leverage collaboration in order to obtain much faster statistical rates, and we therefore need to go beyond the basic estimator $\hat W_{bad}\in \argmin \cL(W)$.
The most natural way to do so is to enforce collaboration \textit{via} some structural regularization, by making local models $w_i$ interdependent. 
Under \Cref{hyp:low_rank}, we therefore want to find the following low-rank estimator, defined as:
\begin{equation}\label{eq:estim_lowrank}
\hat W_{\rm low \, rank}\in\argmin\set{\cL(W)\,\Big|\,\rank(W)\leq r\,,\,\max_{i\in[n]}\NRM{w_i}\leq B}\,,
\end{equation}
while if we work under the clustered assumption \Cref{hyp:clusters}, we are looking for:
\begin{equation}\label{eq:estim_cluster}
       \hat W_{\rm clustered} \in\argmin\Big\{\cL(W_{C,\pi})\,\Big|\, C\in\R^{d\times r}\,,\,\pi:[n]\to [r]\,,
       \max_{i\in[n]}\NRM{w_i}\leq B \Big\}\,,
\end{equation}
where for $C=(c_1|\ldots|c_r)\in\R^{d\times r}$ and $\pi:[n]\to [r]$ we write $W_{C,\pi}\in\R^{d\times n}$ the matrix whose row is $c_{\pi(i)}$.
A first question arising is thus: how do these estimators behave \emph{i.e.}, how many samples $m$ and agents $n$ are required to find the true minimizers $w_i^\star$ ?
We answer these in \Cref{sec:rank,sec:one_sample}
However interesting it might be since it gives an optimistic baseline for our problem, answering this question may not be satisfactory: the estimators defined in \Cref{eq:estim_lowrank,eq:estim_cluster} are solutions of non-convex problems that may be NP hard to compute directly, although many heuristics exist to compute the estimators defined in \Cref{eq:estim_lowrank,eq:estim_cluster}: hard-thresholding algorithms, Bureir-Monteiro factorisations, clustering algorithms, etc.
We will thus resort to tricks that consist in convex relaxations of the problems (\Cref{sec:relaxation}) in order to obtain polynomial-time algorithms.

\section{Statistical bounds on rank-regularized and clustered estimators}\label{sec:rank}

In this section, we first provide statistical upper bounds on the performances of the low rank estimator defined in \Cref{eq:estim_lowrank}.
\citet[Theorem 5]{pmlr-v139-tripuraneni21a} states that for a particular instance of our problem (quadratic linear regression with isotropic gaussian data and noise), for a number of samples per user $m$ and a number of user $n$, any estimator $\hat U$ of $U^\star$ that is based on these $nm$ samples, must satisfy$\dist_\rho^2(\hat U,U^\star)=\Omega\left(\frac{rd}{nm}\right)$.
Hence, under \Cref{hyp:low_rank} we need at least a total number of samples greater than $rd$ in order to learn $U^\star$, and thus at least as much to learn $W^\star$.
The bounds we prove next for the low rank estimator is optimal in the sense that we recover $U^\star$ and $W^\star$ provided that $nm\gg rd$ (up to log factors), but it requires $m$ to be larger than $r$ to be non-trivial. We relax this in \Cref{sec:one_sample}.

\subsection{Generalization properties of the low-rank and clustered estimators}

Although the underlying non-convex optimization problem does not seem possible to solve in polynomial time, efficient approximations exist.
One approach is the Burer-Monteiro factorization~\citep{Burer2003,Burer2004}, which involves parameterizing $A_k$ as $A_k=B_k C_k$ with $B_k\in\R^{d\times r}$ and $C_k\in\R^{r\times d}$. This method relaxes the constraint to a convex set but results in a non-convex function. 
Another approach is \textit{hard-thresholding}, which use projected (S)GD on the non-convex constraint set~\citep{BLUMENSATH2009iht,foucart2018iterative}.

\begin{theorem}\label{thm:low_rank}
	Assume that \Cref{hyp:low_rank} holds (underlying low rank assumption).
	Then, with probability $1- 4e^{-r(d+n)}-4e^{-nmd}$,
	\begin{equation*}
		f(\hat W_{\rm low \, rank})-f(W^\star)= \tilde\cO\left(B^2(\sigma^2+\sigma_\star^2)\sqrt{\frac{r(d+n)}{mn}}\right)\,.
	\end{equation*}
	and, if each $f_i$ is $\mu-$strongly convex,
	\begin{equation*}
		\frac{1}{n}\NRM{\hat W_{\rm low \, rank}-W^\star}_F^2=\tilde\cO\left(B^2(\sigma^4+\sigma_\star^4)\frac{r(d+n)}{\mu^2mn}\right)\,.
	\end{equation*}
\end{theorem}
The low rank estimator hence recovers optimal sample complexity $nm=\tilde\Omega(rd)$ \citep{pmlr-v139-tripuraneni21a}, under the condition $m=\tilde\Omega(r)$, in order to recover good performances. This improves upon all previous works that at best required $nm=\tilde\Omega(r^2d)$ samples \citep{duchi2008efficient,thekumparampil2021statistically,pmlr-v139-tripuraneni21a}, the only exception being the concurrent work \citet{niu2024collaborativelearningsharedlinear} that however only considers an estimator specific to planted linear regression.
In the strongly convex case, we recover all the $w_i^\star$, and thus we can recover the subspace, reaching the optimal $rd$ sample complexity, but with the additional $m\gg r$ condition, which is expected since without this condition, recovering the local heads is impossible. The dependency on $d$ is expected, but can easily be replaced by an \textit{effective dimension} $d_{\rm eff}$ using more refined concentration results for Hessians \citep{pmlr-v134-even21a}.
One question that arises is then: can we hope to recover the underlying subspace even for small $m<r$ ?
We do this next, by further assuming that agents form clusters of agents (this is a particular instance of low rank).
Then, in \Cref{sec:one_sample}, we show that in a special case (noiseless linear regression with quadratic loss), we can recover the subspace even with $m=1$.

Recall that we defined the clustered estimator $\hat W_{\rm clustered}$ in \Cref{eq:estim_cluster},
the underlying optimization problem being still non-convex.
The following Theorem is proved exactly as in the low rank case.

\begin{theorem}\label{thm:cluster}
	Assume that \Cref{hyp:clusters} holds (underlying low rank assumption).
	Then, with probability $1- 4e^{-n}-4e^{-nmd}$, we have $f(\hat W_{\rm clustered})-f(W^\star)= \tilde\cO\left(B^2(\sigma^2+\sigma_\star^2)\sqrt{\frac{rd}{mn}+\frac{1}{m}}\right)$, and if each $f_i$ is $\mu-$strongly convex, we have $\frac{1}{n}\NRM{\hat W_{\rm clustered}-W^\star}_F^2=\tilde\cO\left(B^2(\sigma^4+\sigma_\star^4)\frac{1}{\mu^2}\left(\frac{rd}{mn}+\frac{1}{m}\right)\right)$.
\end{theorem}
In this case, we thus only need $mn\gg rd$ and $m\gg1$ in order to obtain good generalization, therefore bypassing the $m\gg r$ constraint of the lowrank estimator.
However, due to our proof techniques, this was expected.
Indeed, in \Cref{prop:gen}, the uniform upper bounds on the quantities considered scale with $\cW$ the set over which we optimize over: the large this set, the worse the bounds.
Hence, restricting the set from low rank matrices to clustered matrices, we expect better uniform upper bounds.
Quantitatively, these uniform upper bounds depend on the metric entropy of the sets considered \citep{lorentz1966entropy,DUDLEY1974227}.
For low rank matrices as considered in \Cref{thm:low_rank}, the $\eps-$entropy scales as $\cO(r(d+n)\ln(1/\eps))$ \citep{candes2011tight}. Dividing this metric entropy by the number of samples, we obtain $\frac{rd}{nm}+\frac{r}{m}$.
In the case of clustered matrices, we can see the search set as $[r]^{[n]}\times \R^{d\times r}$, leading to an $\eps-$ entropy of $n\ln(r)+ \cO(rd\ln(1/\eps))$. Dividing by the number of samples, we hence obtain a bound of order $\frac{rd}{nm} + \frac{\ln(r)}{m}$.

We next state the ingredients that lead to our generalization bounds on the lowrank and cluster estimators and naturally make appear \Cref{hyp:noise,hyp:hessian}; our noise and Hessian concentration assumptions. Due to the generality of our arguments, we believe they can be of independent interest.

\begin{proof}[Proof sketch for \Cref{thm:low_rank,thm:cluster}]
The estimators we study in this section take the form:
$\hat W\in\argmin_{W\in\cW} \cL(W)$, where we recall that $	\cL(W)=\frac{1}{nm}\sum_{(i,j)\in[n]\times[m]} F_i(w_i,\xi_{ij})\,,\quad W=(w_1|\ldots|w_n)\in\R^{d\times n}$, as defined in \Cref{eq:loss_emp}, and where $\cW$ a subset of $\R^{d\times n}$ that contains $W^\star$ in its interior.
What we can first notice then is that $\cL(\hat W)\leq \cL(W^\star)$ by definition of $W^\star$, leading to:
\begin{equation*}
\begin{aligned}
   D_\cL(\hat W,W^\star)\leq \langle \nabla \cL(W^\star),W^\star-\hat W\rangle\,, 
\end{aligned}
\end{equation*}
where $D_g$ is the Bregman divergence of some function $g$.
Recalling that $\esp{\cL(W)}=f(W)=\frac{1}{n}\sum_{i=1}^nf_i(w_i)$ for some fixed $W$ and that $W^\star$ minimizes $f$, we thus obtain $f(\hat W)-f(W^\star)\leq \langle\nabla \cL(W^\star),W^\star- W\rangle + D_{f-\cL}(W,W^\star)$, leading to $f(\hat W)-f(W^\star)\leq$
\begin{equation*}
\begin{aligned}
  	\sup_{W\in\cW}\left|\langle\nabla \cL(W^\star),W^\star- W\rangle \right|+ \sup_{W\in\cW}\left|D_{f-\cL}(W,W^\star)\right|\,,  
\end{aligned}
\end{equation*}
where we introduce the suprema in order to deal with the fact that $\hat W$ and $\cL$ (the $nm$ samples) are not independent. In order to prove that $\hat W$ performs well, we thus need to control both terms of the RHS. 
The first term is a noise term, since $\esp{\nabla\cL(W^\star)}=\nabla  f(W^\star)=0$, and we will use \Cref{hyp:noise} to control the deviations (uniformly over $\cW$).
For the second term (the Bregman divergence of the difference), noting that $\esp{D_{f-\cL}(W,W^\star)}=D_{0}(W,W^\star)=0$, we will bound the deviations using our Hessian concentration assumption (\Cref{hyp:hessian}), uniformly over $\cW$.
\end{proof}
\subsection{Few-shot learning on a new task and meta-learning of the linear representation} \label{sec:fewshot_meta}

Under both \Cref{hyp:low_rank,hyp:clusters}, there respectively exist matrices $U^\star,C^\star\in\R^{d\times r}$ such that users optima can be written as $w_i^\star=U^\star v_i^\star$ for \textit{local heads} $v_i^\star\in\R^r$ and $w_i^\star= C^\star P^{\star}_i$ for $P^\star\in\R^{r\times n}$ local cluster identification.
In both cases, learning this $d\times r$ matrix simplifies the problem: learning them makes the problem much easier for learning a new task that shares the same structure.
The following result \citep[Lemma 16]{pmlr-v139-tripuraneni21a} shows that when we learn $W^\star$, we indeed learn this shared representation, as expected. See \Cref{app:angle} for principal angle distance definitions.

\begin{lemma}\label{lem:meta-learning}
    Assume that \Cref{hyp:low_rank} holds and that we have access to some $\hat W\in\R^{d\times n}$ of rank at most $r$ such that $\frac{1}{n}\NRM{\hat W-W^\star}^2_F\leq \eps$.
Then, writing $\hat W= \hat U \hat V$ for $\hat U\in\R^{d\times r}$ with orthogonal columns, we have $\frac{\dist_F^2(\hat U,U^\star)}{r}\leq \frac{\eps}{\nu^2}$,
where $\nu^2=\frac{r}{n}\sigma_r(V^\star V^{\star,\top})$ is the smallest eigenvalue of $\frac{r}{n}V^\star V^{\star,\top}=\frac{r}{n} \sum_{i=1}^n v_i^\star v_i^{\star,\top}=\in\R^{r\times r}$.
\end{lemma}
Note that for well-conditioned matrix $W^\star$, we expect $\nu^2=\Omega(1)$ (this is the case for instance if we choose $v_i^\star\sim\cN(0,I_r/r)$).
Hence, if the active $n$ agents learn their local optima, the whole pool of agents has access to the shared representation $U^\star$, up to a small error $\delta$.
Then, a new agent may take advantage of this to perform \textit{few-shot learning}: with only a few samples, the $(n+1)-$th agent may learn its local optimizer, as we show next.

\begin{proposition}\label{prop:fewshot}Assume that \Cref{hyp:low_rank} holds
    and that based on the $nm$ samples $(\xi_{ij})_{i\in[n],j\in[m]}$ of the $n$ users, an estimator $\hat U$ of $U^\star$ satisfying $\dist_\rho^2(\hat U,U^\star)\leq \delta^2$ has been learnt.
    Let a $(n+1)-$th agent have $m'$ samples $(\xi_{n+1,j})_{j\in[m']}$ from some distribution $\cD_{n+1}$ that desires to minimize its local generalization error $f_{n+1}(w)=\esp{F_{n+1}(w,\xi_{n+1})},\xi_{n+1}\sim\cD_{n+1}$.
    Assume that $f_{n+1}$, $F_{n+1}$ satisfy the assumptions of \Cref{sec:hyp} and that $f_{n+1}$ is minimized at some $w_{n+1}^\star=U^\star v_{n+1}^\star$ for some $v_{n+1}^\star\in\R^r$ that satisfies $\NRM{v_{n+1}^\star}\leq B$.
    Then, with probability $1-e^{-m'r}$, the estimator:
    \begin{equation}\label{eq:fewshot_obj}
    \begin{aligned}
       \hat v_{n+1}\in \argmin &\set{ \frac{1}{m'}\sum_{j=1}^{m'} F_{n+1}(\hat U v,\xi_{n+1,j})\,|\,\NRM{v}\leq B }\,,
    \end{aligned}
    \end{equation}
    satisfies $f_{n+1}(\hat U\hat v_{n+1}) - f_{n+1}(W^\star_{n+1})=\frac{L\delta^2}{2}+\tilde\cO\left(B^2(\sigma^2+\sigma_\star^2)\sqrt{\frac{r}{m'}}  \right)$, and if $f_{n+1}$ is $\mu-$strongly convex, $\NRM{\hat v_{n+1} -v^\star_{n+1}}^2=\frac{L\delta^2}{\mu}+\tilde\cO\left(\frac{B^2}{\mu^2}(\sigma^4+\sigma_\star^4)\frac{r}{m'} \right)$.

\end{proposition}
Hence, if a reprensentation has been learnt with enough precision, an incoming agent only needs $m'\gg r$ samples to perform \textit{few-shot learning}.
However, as highlighted by our previous bounds and in particular by \Cref{eq:estim_lowrank}, up to now the representation is only learnt by learning beforehand the whole matrix $W^\star$ of concatenated heads.
The following question thus arises: \textit{can we meta-learn the linear representation $U^\star$ without learning local heads $w_i^\star$ ?}
%

\subsection{Low-rank estimators with only 1 sample per user  }\label{sec:one_sample}

We here study under \Cref{hyp:low_rank} if the linear representation $U^\star$ can be learnt if we only have 1 sample per agent: $m=1$. This case lies in the \textit{meta-learning-without-learning} setting: for each task, it is illusory to want to minimize the local objective $f_i$ and to find $w_i^\star$ wether or not a low dimensional linear representation is learnt; yet learning this representation may help for future tasks, and quite often it is easier to obtain many tasks with few samples per task, than tasks with many samples.
However, this setting is quite challenging, since no concentration result will hold. For presentation sake we focuse on $m=1$ in this section; however, our arguments easily generalize to $m\leq (1-\eps)r$ for $\eps>1/r$.
We still consider the low-rank estimator \Cref{eq:estim_lowrank}, as it is the most natural one under our mild assumptions.
Previous works \citep{pmlr-v139-tripuraneni21a,duchi2022subspace} studied the $m=1$ and $m=2$ cases using Method-of-Moments estimators: such estimators cannot be used in the general convex case.
Even though the linear setting is considered in this section, \textit{our goal is to study general estimators that can be generalized beyond the linear setting}, such as the low-rank estimator (\Cref{eq:estim_cluster}).
To simplify the analysis, we here fall back to studying noisy linear regression with quadratic loss, for which $F_i(w,\xi_{ij}=(x_{ij},y_{ij}))=\frac{1}{2}(\langle x_i,w\rangle -y_i)^2$, where $\xi_{ij}=(x_{ij},y_{ij}))$ for $y_{ij}=\langle w_i^\star,x_{ij}\rangle$ and $x_{ij}\sim\cN(0,I_d)$.
We furthermore assume that 
there exists $\lambda>0$ such that $\frac{1}{n}\sum_{i=1}^n w_i^\star w_i^{\star \top} \geq \frac{\lambda B^2}{r} U^\star U^{\star \top}$, and that all $w_i^\star$ are of norm $B$. For $\lambda=\Omega(1)$ this means that the local heads span all directions of the low-rank subspace.
Broadening our result to more complex settings is left open.

\begin{theorem}\label{thm:m_1}Assume that \Cref{hyp:low_rank} holds.
Let $\hat W$ be any solution of the optimization problem defined in \Cref{eq:estim_lowrank}, and let $\hat U\in\R^{d\times r}$ be the orthogonal projection on its image.
	Let $\eps\in(0,1)$.
	There exist constants $c,C>0$ such that if $n\geq Cd\ln(d)e^{\frac{cr}{\lambda\eps}}$, then $\frac{\dist_F^2(\hat U,U)}{r} \leq \eps$ with probability $1-e^{-rdn}$.
\end{theorem}

The proof of this result differs drastically from that of \Cref{eq:estim_lowrank,thm:cluster} or proofs related such as that of \citet{pmlr-v178-boursier22a,rohde2011}. Indeed, with only one sample per agent, concentration results on the quantities considered (empirical losses mainly) will necessarily fail if we consider all possible local heads. 
Our approach thus only considers shared representations $U$: we compute the probability that a given representation is \textit{admissible}, and extend this to all possible representations.
%
\Cref{thm:m_1} show that it is possible to learn the subspace even in the regime where it is informationnaly impossible to learn local solutions $w_i^\star$.
However, this requires $n>\!\!> de^{cr}$ agents in total.
The following proposition however suggests that this is necessary for the low rank estimator.

\begin{proposition}\label{prop:lower_m_1}
	There exist constants $c',C'>0$ such that if \[n\leq C'e^{c'r}\,,\]
 then the low rank estimator is not consistent: with probability $7/8$ there exist solutions $\hat W$ of the optimization problem \Cref{eq:estim_lowrank} such that the orthogonal projection $\hat U$ on the image of $\hat W$ is \emph{orthogonal} to $U_\star$ \emph{i.e.}, with probability greater than $7/8$:
 \begin{equation*}
     \sup\set{\frac{\dist_F^2({\rm Im} (W),{\rm Im}(U_\star))}{r}\,\Big|\,W\in \cS^\star_n} = 1\,,
 \end{equation*}
 where $\cS^\star_n$ is the set of minimizers of the empirical risk.
\end{proposition}






\section{Nuclear norm regularization}\label{sec:relaxation}


So far, except for the few-shot learning objective \eqref{eq:fewshot_obj}, the optimization problems considered to obtain the shared representation are the minimization of convex functions over \textit{non-convex} constraint sets: low-rank matrices or clustered matrices.
The objective of this section is thus to provide a convex relaxation of \Cref{eq:estim_lowrank} and prove generalization error bounds for the obtained relaxed estimator.
Let, for some $\kappa>0$ that satisfies $\kappa\geq \frac{\lambda_1(W^\star)}{\lambda_r(W^\star)}$ for $\lambda_1(W^\star)$ and $\lambda_r(W^\star)$ largest and $r-$largest singular values of $W^\star$:
\begin{equation}\label{eq:relaxed_estim}
\begin{aligned}
    \hat W \in \argmin &\Big\{ \cL(W)\,|\,W\in\cW\Big\}
\end{aligned}
\end{equation}
where $\cW\subset \R^{d\times n}$ is the set of matrices that satisfy $\sup_i \NRM{w_i}\leq B$ and $\NRM{W}_*\leq \kappa B\sqrt{nr}$.
Let $\hat W_\SVD$ be the top-$s$ SVD of $\hat W$ (the best rank $s$ approximation of $\hat W$ in Frobenius norm), for some $s\in[\min(d,n)]$ to be determined.
As opposed to the previous estimators, the obtained optimization problem consists of a convex objective minimized over a convex set, and can be computed efficiently using \textit{e.g.}, projected gradient descent on $\cW$ (\textit{aka}, iterative soft thresholdings using \citet{duchi2008efficient}) or Frank-Wolfe algorithm \citep{pmlr-v28-jaggi13}.
Importantly, $W^\star$ lies in the optimization set $\cW$.

\begin{theorem}\label{thm:nuc_norm_main}
 Assume that the assumptions of \Cref{thm:low_rank} hold and that each $f_i$ is $\mu$-strongly convex.
With probability $1-4e^{-n}-4e^{-nmd}-e^{-md}$, if $s=\sqrt{r(d+n)}$, the relaxed estimator $\hat W$ satisfies:
\begin{align*}
	&\frac{1}{n}\NRM{\hat W_\SVD-W^\star}_F^2 = \tilde{\cO}\left(\frac{\kappa^2B^2 L}{\mu} \sqrt{\frac{r}{d+n}}+ \frac{B^2((1+\kappa^2)\sigma^2+(1+\kappa)\sigma_\star^2)}{\mu} \frac{\sqrt{r(d+n)}}{m}  \right)\,.
\end{align*}
\end{theorem}
In this rate, the last term dominates, and the sample efficiency appears to be worse than for the non-convex low-rank estimator (\Cref{thm:low_rank}).
Indeed, for the right-hand side to vanish, we require $n\gg d$, which was also the case for the low-rank estimator, but we also need that $m\gg \sqrt{rd}$, as opposed to $m\gg r$ previously: this appears to be the cost of our relaxation. Note that similar rates were obtained by \citet{pmlr-v178-boursier22a} for linear quadratic regression.
The condition $m\gg \sqrt{rd}$ is exactly the geometric mean of the \textit{no-collaboration} rate $m\gg d$ (\Cref{sec:no_collab}) and the idealized \textit{low-rank} rate $m\gg r$: it is a trade-off between statistical and computational efficiency.

\section*{Conclusion and extensions}




In this work, we provided extensions to a well-studied problem: learning low-rank subspaces in multi-task or meta learning. While previous works considered well-specified linear regression, we relaxed this assumption to a more general convex setting, allowing to grasp other important problems, such as classification problems or learning with GLMs.
After studying rank constrained estimators that achieve optimal sample complexity if there is enough sample per task, we studied a nuclear norm relaxation for the rank constraint, and proved that this relaxation perfectly interpolates between the number of samples required by optimal 
rank-regularized estimator and the no-collaboration setting via their geometric mean.

Future works, interesting and challenging extensions include the \textit{non-parametric setting}, where we wish to solve stochastic optimization problems of the form $f^\star\in\argmin_{f\in\cF}\cR(f)$, where $\cF$ is a class of functions $\R^d\to \R$, $\cR(f)=\esp{\ell(Y,f(X)}$ is the risk, for $(X,Y)\sim \cD$ and a loss function $\ell$ \citep{follain2024nonparametriclinearfeaturelearning}. 
Assuming that $f^\star$ can be factorized as $f^\star(x)=g^\star(U^\star x)$, for some $g^\star:\R^r\to \R$ and $U^\star\in\R^{d\times r}$, how can our analysis and results be adapted to this more complex feature learning setting?

\bibliography{refs.bib}
\bibliographystyle{plainnat}

\newpage
\appendix

\section{Proofs of \Cref{sec:rank}}

\subsection{A first general lemma}

\begin{lemma}\label{prop:gen}
	Assume that for some constants $\alpha,\beta>0$ we have  for all $ W\in\cW$:
	\begin{equation*}
		\left|\langle\nabla \cL(W^\star),W^\star- W\rangle \right|\leq \alpha \sigma_\star^2 \sup_{i\in[n]}\NRM{w_i-w_i^\star}_2^2 +\sigma_\star^2\sqrt{\frac{\alpha}{n}\NRM{W-W^\star}_F^2\sup_{i\in[n]}\NRM{w_i-w_i^\star}_2^2}\,,
	\end{equation*}
	and for all $ (W,W')\in\cW^2$,
 \begin{equation*}
		 \left|D_{f-\cL}(W,W')\right|\leq \beta \sigma^2 \sup_{i\in[n]}\NRM{w_i-w_i'}_2^2 +\sigma^2\sqrt{\frac{\beta}{n}\NRM{W-W'}_F^2\sup_{i\in[n]}\NRM{w_i-w_i'}_2^2}\,.	
	\end{equation*}
	Then, if $4B^2\geq \sup_{W,W'\in\cW}\sup_{i\in[n]}\NRM{w_i-w_i'}_2^2$, we have:
	\begin{equation*}
		f(\hat W)-f(W^\star)\leq 4B^2(\alpha\sigma_\star^2+\beta\sigma^2+\sqrt{\alpha}\sigma_\star^2+\sqrt{\beta}\sigma^2)\,.
	\end{equation*}
	Furthermore, if each $f_i$ is $\mu-$strongly convex,
	\begin{equation*}
		\frac{1}{n}\sum_{i=1}^n \NRM{\hat w_i-w_i^\star}^2_2\leq 16\mu^{-1} B^2(\alpha\sigma_\star^2 +\beta\sigma^2) + 64\mu^{-2}B^2 ( \alpha\sigma_\star^4 + \beta\sigma^4)\,.
	\end{equation*}
\end{lemma}

\begin{proof}
    We have:
    \begin{align*}
	f(\hat W)-f(W^\star)&\leq \sup_{W\in\cW}\left|\langle\nabla \cL(W^\star),W^\star- W\rangle \right|+ \sup_{W\in\cW}\left|D_{f-\cL}(W,W^\star)\right|\\
 &\leq 4B^2(\alpha\sigma_\star^2 +\beta\sigma^2) + 2B\sigma_\star^2\sqrt{\frac{\alpha}{n}\NRM{W^\star- \hat W}_F^2}+2B\sigma^2\sqrt{\frac{\beta}{n}\NRM{W^\star- \hat W}_F^2}\,.
\end{align*}
For the first part, $\NRM{W^\star- \hat W}_F^2\leq nB^2$. For the second part, we have $f(\hat W)-f(W^\star)\geq \frac{\mu}{2 n}\NRM{W^\star- \hat W}_F^2$, leading to
\begin{align*}
    \frac{1}{ n}\NRM{W^\star- \hat W}_F^2 & \leq 8\mu^{-1} B^2(\alpha\sigma_\star^2 +\beta\sigma^2) + 4\mu^{-1}B (\sqrt \alpha\sigma_\star^2 +\sqrt \beta\sigma^2)\sqrt{\frac{1}{ n}\NRM{W^\star- \hat W}_F^2}\,.
\end{align*}
Using $x^2\leq a x + b$ and $x\geq 0$ implies $x^2\leq 2a^2+ 2b$, we obtain that 
\begin{align*}
    \frac{1}{ n}\NRM{W^\star- \hat W}_F^2 & \leq 16\mu^{-1} B^2(\alpha\sigma_\star^2 +\beta\sigma^2) + 32\mu^{-2}B^2 (\sqrt \alpha\sigma_\star^2 +\sqrt \beta\sigma^2)^2\\
    & \leq 16\mu^{-1} B^2(\alpha\sigma_\star^2 +\beta\sigma^2) + 64\mu^{-2}B^2 ( \alpha\sigma_\star^4 + \beta\sigma^4)
\end{align*}
\end{proof}

\subsection{Concentration bounds}

\begin{lemma}[Noise concentration]\label{lem:noise_concentration}
	Let $\Delta\in \R^{d\times n},t>0$. We have:
	\begin{equation*}
		\proba{\left|\langle\nabla \cL(W^\star),\Delta\rangle \right| >t} \leq 2\exp\left(-c_1\min\left(\frac{nmt}{\sigma_\star^2 \sup_{i\in[n]}\NRM{\Delta_i}^2},\frac{nmt^2}{\frac{\sigma_\star^4 }{n}\NRM{\Delta}_F^2\sup_{i\in[n]}\NRM{\Delta_i}^2}\right)\right)\,.
	\end{equation*}
\end{lemma}
\begin{proof}
	Remark that we have: 
	\begin{equation*}
		\langle \nabla \cL(W^\star),\Delta\rangle=\frac{1}{nm}\sum_{i,j}\langle \nabla_w F_i(w_i^\star,\xi_{ij}),\Delta_i\rangle\,,
	\end{equation*}
	and using \Cref{hyp:noise}, $\langle \nabla_w F_i(w_i^\star,\xi_{ij}),\Delta_i\rangle$ are \emph{i.i.d.} $\sigma_\star^2\sup_i\NRM{\Delta_i}^2$ subexponential centered random variables.
	By Hanson-Wright inequality \citep{hansonwright,hanson_vershi}, we have our result.
\end{proof}

\begin{lemma}\label{lem:breg_concentration}
	Let $W,W'\in\R^{d\times n},t>0$. We have:
	\begin{align*}
		\proba{\left|D_{f-\cL}(W,W')\right|>t}&\leq \frac{8H\sup_{i\in[n]}\NRM{w_i-w_i'}}{t}\\
&\times\exp\left(-c_2\min\left(\frac{nmt}{\sigma^2 \sup_{i\in[n]}\NRM{w_i-w_i'}^2},\frac{nmt^2}{\frac{\sigma^4 }{n}\NRM{W-W'}_F^2\sup_{i\in[n]}\NRM{w_i-w_i'}^2}\right)\right) \,.
	\end{align*}
\end{lemma}
\begin{proof}
	Using Taylor's second order equality, there exists $\tilde W=\lambda W+(1-\lambda)W'\in [W,W']$ (for $\lambda\in[0,1]$) such that
	\begin{align*}
		D_{\hat\cL_{n,m}-f}(W,W')&=\frac{1}{2}\NRM{W-W'}^2_{\nabla^2(\hat\cL_{n,m}-f)(\tilde W)}\\
		&=\frac{1}{nm}\sum_{ij} (w_i-w_i')^\top 
		\big(\nabla^2_{ww}F_i(\tilde w_i,\xi_{ij})-\nabla^2 f_i(\tilde w_i)\big) (w_i-w_i')\,.
	\end{align*}
	The subtlelty here lies in the fact that $\tilde W$ depends on the samples $\xi_{ij}$, preventing us from directly using concentration results.
	We will thus prove concentration of 
	\begin{equation*}
		\frac{1}{nm}\sum_{ij} (w_i-w_i')^\top 
		\big(\nabla^2_{ww}F_i(\Delta_i,\xi_{ij})-\nabla^2 f_i(\Delta_i)\big) (w_i-w_i')\,,
	\end{equation*}
	for some fixed $\Delta_i\in\R^{d\times n}$.
	$(w_i-w_i')^\top \nabla^2_{ww}F_i(\Delta_i,\xi_{ij})(w_i-w_i')$ is $\NRM{w_i-w_i'}^2\sigma^2-$subexponential (\Cref{hyp:hessian}) and thus, by a centering Lemma, since $\nabla^2 f_i(\Delta_i)\big) (w_i-w_i')=\esp{(w_i-w_i')^\top \nabla^2_{ww}F_i(\Delta_i,\xi_{ij})(w_i-w_i')}$, $(w_i-w_i')^\top 
	\big(\nabla^2_{ww}F_i(\Delta_i,\xi_{ij})-\nabla^2 f_i(\Delta_i)\big) (w_i-w_i')$ is a $c\NRM{w_i-w_i'}^2\sigma^2-$subexponential random variable. Using Hanson-Wright inequality \citep{hansonwright}:
	\begin{align*}
		&\proba{\left|\frac{1}{nm}\sum_{ij} (w_i-w_i')^\top 
		\big(\nabla^2_{ww}F_i(\Delta_i,\xi_{ij})-\nabla^2 f_i(\Delta_i)\big) (w_i-w_i')\right| >t}\\
		&\qquad\leq 2\exp\left(-c_2\min\left(\frac{nmt}{\sigma^4 \sup_{i\in[n]}\NRM{w_i-w_i'}^2},\frac{nmt^2}{\frac{\sigma^4 }{n}\NRM{W-W'}_F^2\sup_{i\in[n]}\NRM{w_i-w_i'}^2}\right)\right)  \,.
	\end{align*}
	We however want such a bound over $\sup_{\Delta\in[W,W']}\left|\frac{1}{nm}\sum_{ij} (w_i-w_i')^\top 
	\big(\nabla^2_{ww}F_i(\Delta_i,\xi_{ij})-\nabla^2 f_i(\Delta_i)\big) (w_i-w_i')\right|$.
	The quantity in the sup here is $2H\sup_{i\in[n]}\NRM{w_i-w_i'}$ Lipschitz in $\Delta_i$.
	Since $\set{ W+\lambda(W'-W), \lambda\in[0,1] } = [W,W']$,taking an $\eps-$net of $[0,1]$ (of size $\leq 2/\eps$), 
	\begin{align*}
		&\proba{\forall \Delta\in[W,W']\,,\quad\left|\frac{1}{nm}\sum_{ij} (w_i-w_i')^\top 
		\big(\nabla^2_{ww}F_i(\Delta_i,\xi_{ij})-\nabla^2 f_i(\Delta_i)\big) (w_i-w_i')\right| >t+2H\sup_{i\in[n]}\NRM{w_i-w_i'}\eps}\\
		&\qquad\leq \frac{4}{\eps}\exp\left(-c_2\min\left(\frac{nmt}{\sigma^4 \sup_{i\in[n]}\NRM{w_i-w_i'}^2},\frac{nmt^2}{\frac{\sigma^4 }{n}\NRM{W-W'}_F^2\sup_{i\in[n]}\NRM{w_i-w_i'}^2}\right)\right)  \,.
	\end{align*}
	Thus, for $\eps=\frac{t}{2H\sup_{i\in[n]}\NRM{w_i-w_i'}}$, we have our result.

\end{proof}

\subsection{Proof of \Cref{thm:low_rank}}

\begin{proof}
	Let $\cW=\set{W\in\R^{d\times n}\,|\,\rank(W)\leq r,\max_{i\in[n]}\NRM{w_i}\leq B}$.
	We begin with the two following lemmas, from which \Cref{thm:low_rank} is directly derived.
	\begin{lemma}\label{lem:concentration_noise_unif_lowrank}
		With probability $1-2e^{r(d+n)}-2e^{nmd}$, we have that, for all $\Delta\in\cW$,
		\begin{align*}
			\left|\langle\nabla \cL(W^\star),\Delta\rangle \right|\leq C_1 \sigma_\star^2\frac{r(d+n)\ln\left( \frac{\sigma_\star^2nmd}{r(d+n)} \right)  }{nm}\sup_{i\in[n]} \NRM{\Delta_i}^2+ C_1\sigma_\star^2 \sqrt{\frac{r(d+n)\ln\left( \frac{\sigma_\star^2nmd}{r(d+n)} \right)}{nm}\frac{1}{n}\NRM{\Delta}_F^2\sup_{i\in[n]} \NRM{\Delta_i}^2}\,.
		\end{align*}
	\end{lemma}

	\begin{proof}[Proof of \Cref{lem:concentration_noise_unif_lowrank}]
		Let $\cW'=\set{W\in\R^{d\times n}\,|\,\rank(W)\leq r,\max_{i\in[n]}\NRM{w_i}= B}$ (by homogeneity of the result, we can impose $\max_{i\in[n]}\NRM{w_i}= B$).
		Let $\cW_\eps$ be an $\eps-$net of $\cW'$: we know that we can have $|\cW_\eps|\leq e^{cr(d+n)\ln(B/\eps)}$.
		Using \Cref{lem:noise_concentration}, 
		\begin{align*}
			\proba{\left|\langle\nabla \cL(W^\star),\Delta\rangle \right|\geq C_1' \sigma_\star^2t\sup_{i\in[n]} \NRM{\Delta_i}^2+ C_1' \sigma_\star^2 \sqrt{\frac{t}{n}\NRM{\Delta}_F^2\sup_{i\in[n]} \NRM{\Delta_i}^2}}\leq 2\exp\left(-nmt\right)\,.
		\end{align*}
		Thus,
		\begin{align*}
			\proba{\exists \Delta\in\cW_\eps\,,\,\left|\langle\nabla \cL(W^\star),\Delta\rangle \right|\geq C_1' \sigma_\star^2t\sup_{i\in[n]} \NRM{\Delta_i}^2+ C_1' \sigma_\star^2 \sqrt{\frac{t}{n}\NRM{\Delta}_F^2\sup_{i\in[n]} \NRM{\Delta_i}^2}}\leq 2\exp\left(-nmt+cr(d+n)\ln(B/\eps)\right)\,.
		\end{align*}
		Then, $\left|\langle\nabla \cL(W^\star),\Delta\rangle \right|$ is $\sup_{i,j}\NRM{\nabla F_i(w_i^\star,\xi_{ij})}$-Lipschitz in $\Delta$. Since $\nabla F_i(w_i^\star,\xi_{ij})$ are independent $\sigma_\star^2$ (multivariate) subexponential random variables.
		Thus, with probability $1-2e^{-nmd}$, we have $\sup_{i,j}\NRM{\nabla F_i(w_i^\star,\xi_{ij})}\leq \sigma_\star^2d \ln(nm)$.
		Hence, using that $\cW_\eps$ is and $\eps-$net of $\cW$,
		\begin{align*}
			&\proba{\exists \Delta\in\cW'\,,\,\left|\langle\nabla \cL(W^\star),\Delta\rangle \right|\geq \eps\sigma_\star^2d \ln(nm)+ C_1' \sigma_\star^2t\sup_{i\in[n]} \NRM{\Delta_i}^2+ C_1' \sigma_\star^2 \sqrt{\frac{t}{n}\NRM{\Delta}_F^2\sup_{i\in[n]} \NRM{\Delta_i}^2}}\\
			&\quad\leq 2\exp\left(-nmt+cr(d+n)\ln(B/\eps)+2e^{-nmd}\right)\,.
		\end{align*}
		We conclude by taking $t=\frac{cr(d+n)(1+\ln(B/\eps))}{nm}$ for $\eps=\frac{1}{\sigma_\star^2d\ln(nm)}\frac{Br(d+n)}{nm}$.
	\end{proof}

	\begin{lemma}\label{lem:concentration_div_unif_lowrank}
		With probability $1-2e^{r(d+n)}-2e^{nmd}$, we have that, for all $(W,W')\in\cW^2$,
		\begin{align*}
			&\left|D_{f-\cL}(W,W')\right|\\
			&\leq C_2 \sigma^2 \sup_{i\in[n]}\NRM{w_i-w_i'}^2 \ln\left( \frac{nm(LB+\sigma_\star^2nmd)}{\sigma^2Br(d+n)} \right)\frac{r(d+n)}{nm}\ln\left(\frac{HBnm}{r(d+n)}\right)\\
			&\quad + C_2\sigma^2\sqrt{  \frac{1 }{n}\NRM{W-W'}_F^2\sup_{i\in[n]}\NRM{w_i-w_i'}^2 \ln\left( \frac{nm(LB+\sigma_\star^2nmd)}{\sigma^2Br(d+n)} \right)\frac{r(d+n)}{nm}\ln\left(\frac{HBnm}{r(d+n)}\right)}\,.
		\end{align*}
	\end{lemma}
	\begin{proof}[Proof of \Cref{lem:concentration_div_unif_lowrank}]
	Let $\cW'=\set{(W,W')\in\cW, \sup_i\NRM{w_i-w'_i}=B}$.
	Using \Cref{lem:breg_concentration},
	\begin{align*}
		&\P\left[\left|D_{f-\cL}(W,W')\right|>C_2 \sigma^2 \sup_{i\in[n]}\NRM{w_i-w_i'}^2 t\ln\left(\frac{H\sup_{i\in[n]}\NRM{w_i-w_i'}}{t}\right)\right.\\
		&\qquad \left. + C_2\sigma^2\sqrt{  \frac{\sigma^4 }{n}\NRM{W-W'}_F^2\sup_{i\in[n]}\NRM{w_i-w_i'}^2 t\ln\left(\frac{H\sup_{i\in[n]}\NRM{w_i-w_i'}}{t}\right)} \quad \right]\\
		&\quad\leq 2\exp\left(-nmt\right) \,.
	\end{align*}
	Hence,
	\begin{align*}
		&\P\left[\exists(W,W')\in\cW_\eps^2\,,\,\left|D_{f-\cL}(W,W')\right|>C_2 \sigma^2 \sup_{i\in[n]}\NRM{w_i-w_i'}^2 t\ln\left(\frac{H\sup_{i\in[n]}\NRM{w_i-w_i'}}{t}\right)\right.\\
		&\qquad \left. + C_2\sigma^2\sqrt{  \frac{\sigma^4 }{n}\NRM{W-W'}_F^2\sup_{i\in[n]}\NRM{w_i-w_i'}^2 t\ln\left(\frac{H\sup_{i\in[n]}\NRM{w_i-w_i'}}{t}\right)} \quad \right]\\
		&\quad\leq 2\exp\left(-nmt+2cr(d+n)\ln(B/\eps)\right) \,.
	\end{align*}
	$\left|D_{f-\cL}(W,W')\right|$ is $(LB+\sup_{i,j}\NRM{\nabla F_i(w_i^\star,\xi_{ij})})-$Lipschitz in $W$ and in $W'$, so that, since $\cW_\eps$ is an $\eps-$net of $\cW$,
	\begin{align*}
		&\P\left[\exists(W,W')\in\cW'\,,\,\left|D_{f-\cL}(W,W')\right|>C_2 \sigma^2 \sup_{i\in[n]}\NRM{w_i-w_i'}^2 t\ln\left(\frac{H\sup_{i\in[n]}\NRM{w_i-w_i'}}{t}\right)\right.\\
		&\qquad \left. + 2\eps (LB+\sup_{i,j}\NRM{\nabla F_i(w_i^\star,\xi_{ij})})+ C_2\sigma^2\sqrt{  \frac{\sigma^4 }{n}\NRM{W-W'}_F^2\sup_{i\in[n]}\NRM{w_i-w_i'}^2 t\ln\left(\frac{H\sup_{i\in[n]}\NRM{w_i-w_i'}}{t}\right)} \quad \right]\\
		&\quad\leq 2\exp\left(-nmt+2cr(d+n)\ln(B/\eps)\right) \,.
	\end{align*}
	Then, since with probability $1-2e^{-nmd}$ we have $\sup_{i,j}\NRM{\nabla F_i(w_i^\star,\xi_{ij})}\leq \sigma_\star^2nmd$, taking $t = \frac{(2c\ln(B/\eps)+1)r(d+n)}{nm} $ and $\eps=\frac{\sigma^2B^2r(d+n)}{nm}\frac{1}{LB+\sup_{i,j}\NRM{\nabla F_i(w_i^\star,\xi_{ij})}}$, we have ou result.
	\end{proof}
	Finally, we conclude the proof using \Cref{prop:gen}.
\end{proof}

\subsection{Proof of \Cref{thm:cluster}}

\begin{proof}
	Same as \Cref{thm:low_rank}, but the $\eps-$net cardinality scales as $r^n\times e^{crd\ln(1/\eps)}$.
\end{proof}

\subsection{Proof of \Cref{prop:fewshot}}

Let $g(v)=\frac{1}{m'}\sum_{j=1}^{m'} F_{n+1}(\hat U v,\xi_{n+1,j})$ and $G(v)=f_{n+1}(\hat U v)$.
We have the bias-variance decomposition, where $v'\in\argmin\set{f_{n+1}(\hat U v)|\NRM{v}\leq B}$
\begin{align*}
    f_{n+1}(\hat U\hat v_{n+1})-f_{n+1}(w_{n+1}^\star)&\leq f_{n+1}(\hat U v')-f_{n+1}(U^\star v_{n+1}^\star) + f_{n+1}(\hat U\hat v_{n+1}) -f_{n+1}(\hat U v')\,.
\end{align*}
Using optimality and then smoothness and gradients null at the optimum, the first term satisfies 
\begin{align*}
    f_{n+1}(\hat U v')-f_{n+1}(U^\star v_{n+1}^\star)&\leq f_{n+1}(\hat U v_{n+1}^\star)-f_{n+1}(U^\star v_{n+1}^\star) \\
    &\leq \frac{L}{2}\NRM{\hat U v_{n+1}^\star-U^\star v_{n+1}^\star}^2\\
    &\leq \frac{L\delta^2}{2}\,.
\end{align*}
Then, the second term can be bounded as in the previous proofs.
Using optimality, $G(\hat v_{n+1})-G(v')\leq 0$ implies that $g(\hat v_{n+1})-g(v')\leq \langle \nabla G(v'),\hat v_{n+1}-v'\rangle + D_{g-G}(v_{n+1},v')$. Mimicking the proof of \Cref{thm:low_rank} and replacing $\eps-$nets by those of the $B-$ball of $\R^r$ (that scale as $\exp(cr\ln(B/\eps))$), we obtain the desired result.

\section{Proof of \Cref{thm:m_1} and \Cref{prop:lower_m_1}}

\begin{proof}
In this proof, we assume that there is only one sample per task: $m=1$. To simplify notations, we thus write $y_i=\langle w_i^\star,x_i\rangle$.
The following notion of ‘‘admissibility'' for some orthogonal matrix is then introduced: a matrix is said admissible if it satisfies the same properties as $U^\star$.

\begin{definition}
	We say that some orthogonal matrix $U\in\R^{d\times r}$ is $\delta$-admissible for $\delta\in(0,1]$ if there exists $v_1,\ldots,v_n\in\R^r$ such that for all $i\in[n]$ we have $y_i=\langle Uv_i,x_i\rangle$ and $\delta\NRM{v_i}\leq B$\,.
\end{definition}

We now compute the probability that some given matrix $U$ is admissible.

\begin{proposition}\label{prop:proba_admissibility}
	Assume that the Gaussian random design holds.
	Let $U\in\R^{d\times r}$ be an orthogonal matrix, let $\eps=\dist_F^2(U,U^\star)$, and assume that $\eps>0$.
	Furthermore, assume that (i) for all $i\in[n]$, $\NRM{w_i^\star}=B$ and (ii) there exists $\lambda>0$ such that $\frac{1}{n}\sum_{i=1}^n w_i^\star w_i^{\star \top} \geq \frac{\lambda B^2}{r} U^\star U^{\star \top}$.
	Then there exist constants $c_1,c_2>0$ such that if $\delta \geq 1-\frac{\lambda \eps}{4}$, we have:
	\begin{equation*}
		\proba{U\text{ is } \delta\text{-admissible}} \leq \exp\left( -c_1 n e^{-\frac{c_2 r}{\lambda\eps}}\right)\,.
	\end{equation*}
\end{proposition}

\begin{proof}
	Notice the following.
	\begin{lemma}
		We have the following equivalence:
		\begin{equation*}
			U\text{ is } \delta\text{-admissible} \quad \Longleftrightarrow \quad \forall i\in[n]\,,\quad \NRM{U^\top x_i}_2 \geq \frac{\delta|y_i|}{B}\,.
		\end{equation*}
	\end{lemma}
	Thus, $\proba{U\text{ is } \delta\text{-admissible}}=\prod_{i=1}^n \proba{\NRM{U^\top x_i}_2 \geq \frac{\delta|y_i|}{B}}$, by independence of the samples.
	We have $y_i=\langle Uv_i^\star,x_i\rangle$, leading us to consider the decomposition $U^\top = U^\top P_i^\star + U^\top (I_d-P_i^\star)$, where $P_i^\star=p_i^\star p_i^{\star \top}\in\R^{d\times d}$ is the orthogonal projector onto the span of $w_i^\star$, for $p_i^\star=\frac{w_i^\star}{\NRM{w_i^\star}_2}$.
	Thus, $\NRM{U^\top x_i}_2=\NRM{U^\top P_i^\star x_i + U^\top (I_d-P_i^\star)x_i}_2\leq \NRM{U^\top P_i^\star x_i}_2 + \NRM{U^\top (I_d-P_i^\star)x_i}_2$.
	We now assume for simplicity that $B=1$.  We have:
	\begin{align*}
		\proba{\NRM{U^\top x_i}_2 \geq \frac{\delta|y_i|}{B}} &\leq \proba{\NRM{U^\top (I_d-P_i^\star)x_i}_2 \geq \delta|\langle w_i^\star,x_i\rangle|-\NRM{U^\top P_i^\star x_i}_2}\\
		&\leq \proba{\NRM{U^\top (I_d-P_i^\star)x_i}_2 \geq \delta|\langle w_i^\star,x_i\rangle|-\NRM{U^\top p_i^\star}_2|\langle p_i^\star,x_i\rangle|}\\
		&= \proba{\NRM{U^\top (I_d-P_i^\star)x_i}_2 \geq \left(\delta-\frac{\NRM{U^\top p_i^\star}_2}{\NRM{w_i^\star}_2}\right)|\langle w_i^\star,x_i\rangle|}\,,
	\end{align*}
	since $U^\top P_i^\star x_i = \langle p_i^\star , x_i\rangle U^\top p_i^\star = \frac{1}{\NRM{w_i^\star}}\langle w_i^\star , x_i\rangle U^\top p_i^\star$.

	We now notice that $\NRM{U^\top (I_d-P_i^\star)x_i}_2 $ and $|\langle w_i^\star,x_i\rangle|$ are independent random variables, using the Gaussian random design assumption and the fact that $w_i^\star$ is by definition in the orthogonal of the span of $(I_d-P_i^\star)$.
	We have $U^\top (I_d-P_i^\star)x_i\sim \gau{U^\top (I_d-P_i^\star) U}$ and since $U^\top (I_d-P_i^\star) U\preceq U^\top U = I_r$, the random variable $\NRM{U^\top (I_d-P_i^\star)x_i}_2^2$ is stochastically dominated by $Z_i$ a $\chi^2_r$ random variable of dimension $r$ independent of $\langle w_i^\star,x_i\rangle$.
	Then, writing $z_i=\langle w_i^\star,x_i\rangle^2$ (a $\chi_1^2$ random variable), we have that $z_i$ and $Z_i$ are independent, and:
	\begin{equation*}
		\proba{\NRM{U^\top x_i}_2 \geq \delta|y_i|} \leq \proba{Z_i\geq \alpha_i z_i}\,,
	\end{equation*}
	where $\alpha_i=\max\left(0,\delta-\frac{\NRM{U^\top p_i^\star}_2}{\NRM{w_i^\star}_2}\right)^2$.
	To upper-bound $\proba{Z_i\geq \alpha_i z_i}$, we lower bound $\proba{Z_i\leq \alpha_i z_i}$. For $Z_i=a_1+\ldots+a_r$ where $(a_i)$ are independent $\chi_1^2$ random variables,
	\begin{align*}
		\proba{Z_i\leq \alpha_i z_i}&\geq \proba{Z_i\leq r/2, \alpha_i z_i\geq r/2}\\
		& = \proba{Z_i\leq r/2} \proba{\alpha_i z_i\geq r/2}\\
		& \geq \proba{a_1\leq 1/2}^r \proba{\alpha_i z_i\geq r/2}\\
		& \geq \proba{a_1\leq 1/2}^r \frac{2}{r\alpha_i}e^{-\frac{r}{4\alpha_i}} \\
		& \geq \frac{2}{r\alpha_i}e^{-\frac{c'r}{\alpha_i}} \\
		&\geq e^{-\frac{cr}{\alpha_i}}\,.
	\end{align*}
	Thus,
	\begin{equation*}
		\proba{\NRM{U^\top x_i}_2 \geq \frac{\delta|y_i|}{B}} \leq 1 - e^{-\frac{cr}{\alpha_i}}\,,
	\end{equation*}
	and
	\begin{equation*}
		\proba{U\text{ is } \delta\text{-admissible}} \leq \prod_{i=1}^n \big(1 - e^{-\frac{cr}{\alpha_i}}\big)\,,
	\end{equation*}
	for $\alpha_i=\max\left(0,\delta-\frac{\NRM{U^\top p_i^\star}_2}{\NRM{w_i^\star}_2}\right)^2$.

	In order to use this without too much trouble, we now make the following assumptions about the norm of each vector $w_i^\star$ and about the diversity of the direction they span: \emph{(i)} for all $i\in[n]$, $\NRM{w_i^\star}_2=1$ and \emph{(ii)} $\frac{1}{n}\sum_{i=1}^n p_i^\star p_i^{\star \top}\preceq \frac{L}{r} U^\star U^{\star \top }$ for some $\lambda>0$.

	Our goal is now to lower bound the $\alpha_i$'s.
	Using \emph{(i)}, we can simplify each $\alpha_i$ as $\alpha_i=\max\left(0,\delta-\NRM{U^\top p_i^\star}_2\right)^2$.
	Then, we have 

	\begin{align*}
		\frac{1}{n}\sum_{i=1}^n \NRM{U^\top p_i^\star}_2^2 &= \frac{1}{n}\sum_{i=1}^n p_i^{\star \top} U U^\top p_i^\star \\
		&= \frac{1}{n}\sum_{i=1}^n p_i^{\star \top} p_i^\star -\frac{1}{n}\sum_{i=1}^n p_i^{\star \top} \big(U^\star U^{\star \top} - U U^\top\big) p_i^\star \\
		&= 1-\frac{1}{n}\sum_{i=1}^n \Tr\Big(p_i^{\star \top} \big(I_d - U U^\top\big) p_i^\star\Big) \\
		&= 1-\frac{1}{n}\sum_{i=1}^n \Tr\Big(p_i^\star p_i^{\star \top} \big(I_d - U U^\top\big) \Big) \\
		&= 1-\Tr\Big(\frac{1}{n}\sum_{i=1}^n p_i^\star p_i^{\star \top} \big(I_d - U U^\top\big) \Big) \\
		&\leq 1-\frac{\lambda}{r} \Tr\Big(U^\star U^{\star \top} \big(I_d - U U^\top\big) \Big) \\
		&= 1-\frac{\lambda}{r} \Tr\Big( \big(I_d - U U^\top\big)U^\star U^{\star \top} \Big) \\
		&= 1-\frac{\lambda}{r} \dist_F^2(U,U^\star) \\
		&= 1-\lambda\eps\,,
	\end{align*}
	where we used $\dist_F^2(U,U^\star)= \NRM{(I-P)P^\star}_F^2=\Tr((I-P)P^\star((I-P)P^\star)^\top)=\Tr((I-P)P^\star(I-P))=\Tr((I-P)P^\star)$, where $P=UU^\top$ and $P^\star=U^\star U^{\star \top}$.
	Hence, since for all $i$ we have $\NRM{U^\top p_i^\star}_2^2\leq 1$, using $\frac{1}{n}\sum_{i=1}^n \NRM{U^\top p_i^\star}_2^2\leq 1-\lambda\eps$, we have that:
	\begin{equation*}
		1 - \lambda \eps \geq \frac{1}{n}\sum_{i=1}^n \NRM{U^\top p_i^\star}_2^2 \geq (1-\lambda\eps/2)\frac{|\set{i:\NRM{U^\top p_i^\star}_2^2\geq 1-\lambda\eps/2}|}{n}=(1-\lambda\eps/2)\left(1-\frac{|\set{i:\NRM{U^\top p_i^\star}_2^2\leq 1-\lambda\eps/2}|}{n}\right)\,,
	\end{equation*}
	so that $|\set{i:\NRM{U^\top p_i^\star}_2^2\leq 1-\lambda\eps/2}| \geq \Big(1 -\frac{1-\lambda\eps}{1-\lambda\eps/2}\Big) n= \frac{\lambda\eps/2}{1-\lambda\eps/2} n\geq \frac{\lambda\eps n}{2}$: at least $\frac{\lambda\eps n}{2}$ of the $\alpha_i$'s are thus smaller than $\max(0,\frac{\lambda\eps}{2} - (1-\delta))\leq \frac{\lambda\eps}{4}$ if $\delta \geq 1-\frac{\lambda\eps}{4}$.
	This leads to:
	\begin{equation*}
		\proba{U\text{ is } \delta\text{-admissible}}  \leq \prod_{i=1}^n \big(1 - e^{-\frac{cr}{\alpha_i}}\big) \leq \left( 1 - e^{-\frac{C'r}{\lambda\eps}} \right)^{\frac{\lambda\eps n}{2}}\leq \exp\left( -c_1 n e^{-\frac{c_2 r}{\lambda\eps}}\right) \,.
	\end{equation*}
\end{proof}

Now that we have proved \Cref{prop:proba_admissibility}, we can state our feasibility bound on the number of tasks when there is only one sample per task.
The estimator we build is quite simple. We define it as any $\hat U$ that satisfies
\begin{equation*}
	\hat U \in\set{U\in\R^{d\times r}\text{ orthonormal such that } \forall i\in[n]\,,\quad B\NRM{U^\top x_i}\geq |y_i|}\,,
\end{equation*}
and this set is not empty since it contains $U^\star$.
In fact, any $\hat U$ built from the image of some
\begin{equation*}
	\hat W\in\argmin\set{\cL(W)\,,\quad \rank(W)\leq r\,,\quad \NRM{w_i}_2\leq B}\,,
\end{equation*}
is in this set.

We finally conclude by proving our theorem.
	Let $\cU=\set{U\in\R^{d\times r}:U\text{ is orthonormal}, \dist_F^2(U,U^\star)\geq \eps}$ and let $\cU_\eta$ be an $\eta$-net of $\cU$ of minimal cardinality.
	We know that we have $|\cU_\eta|\leq e^{crd\ln(1/\eta)}$.
	For $U\in\cU$, there exists $U_\eta\in\cU_\eta$ such that $\NRM{U-U_\eta}\leq \eta$, and for any $i\in[n]$,
	\begin{equation*}
		\NRM{U^\top x_i}\leq \NRM{U_\eta^\top x_i} + \NRM{(U-U_\eta)^\top x_i}\leq  \NRM{U_\eta^\top x_i} + \eta\NRM{ x_i}\,.
	\end{equation*}
	Thus, if $U_\eta$ is not $\delta$-admissible for some $\delta$, there exists $i$ such that $\NRM{U_\eta^\top x_i}\leq \delta |y_i|$, and thus
	\begin{equation*}\label{eq:proof_feasibility_one_1}
		\NRM{U^\top x_i}\leq \delta |y_i| + \eta\NRM{ x_i} < |y_i|\,,
	\end{equation*}
	provided that $\eta < \frac{|y_i|}{\NRM{x_i}}(1-\delta)$.

	Let thus $\delta=1-\frac{\lambda\eps}{4}$ and $\eta < \min_{i\in[n]} \frac{|y_i|}{\NRM{x_i}}(1-\delta)$. With high probability, we have $\eta$ of order $\frac{1}{\sqrt{d}}\lambda\eps$ (up to constant and log factors).
	As above, let $\cU_\eta$ be an $\eta$-net of $\cU$ (of cardinality less than $e^{crd\ln(1/\eta)}$).
	Using an union bound over $\cU_\delta$ and \cref{prop:proba_admissibility}:
	\begin{align*}
		\proba{\exists U\in\cU_\delta \text{ admissible}}&\leq |\cU_\delta|\exp\left(-c_1ne^{-\frac{c_2r}{\lambda\eps}}\right)\\
		&\leq \exp\left(crd\ln(1/\delta)-c_1ne^{-\frac{c_2r}{\lambda\eps}}\right)\,.
	\end{align*}
	Hence, if $n\geq \frac{e^{\frac{c_2r}{\lambda\eps}}}{c_1} rd(1+c\ln(1/\delta))$, we have that with probability $1-e^{-rdn}$, all elements of $\cU_\delta$ are not admissible.
	
	We now place ourselves on this event of probability $1-e^{-rnd}$.
	Let $U \in \cU$. In light of \cref{eq:proof_feasibility_one_1}, there exists $i\in[n]$ such that $\NRM{U^\top x_i}\leq y_i$: $U$ is not admissible.
	Since $U^\star$ is $1$-admissible, the set $\set{U\in\R^{d\times r}\text{ orthonormal and admissible}}$ is non empty. 
	Let thus $\hat U$ be any matrix satisfying
	\begin{equation*}
		\hat U \in\set{U\in\R^{d\times r}\text{ orthonormal such that } \forall i\in[n]\,,\quad \NRM{U^\top x_i}\geq |y_i|}\,.
	\end{equation*}
	Since this set does not intersect $\cU$, we have $\frac{1}{r}\dist_F^2(\hat U, U^\star)\leq \eps$, concluding our proof.
\end{proof}

We now prove \Cref{prop:lower_m_1}.

\begin{proof}
    Without loss of generality, assume that $B=1$.
    Let $U\in\R^{d\times r}$ with orthogonal columns such that the span of $U$ is in the orthogonal of the span of $U^\star$.
    We have that:
    \begin{equation*}
        \proba{U{\rm\, is \,admissible}}=\proba{\forall i\in[n], \NRM{U^\top x_i}_2^2\geq y_i^2}\,.
    \end{equation*}
    Since span of $U$ is in the orthogonal of the span of $U^\star$, $\NRM{U^\top x_i}_2^2$ and $y_i^2$ are independent random variables, of respective laws $\chi_r^2$ and $\chi_1^2$.
    This leads to:
    \begin{align*}
        \proba{U{\rm\, is \,admissible}}&=\proba{\chi_r^2\geq\chi_1^2}^n\\
        &\geq \proba{\chi_r^2\geq r/2}^n\proba{\chi_1^2\leq r/2}\,.
    \end{align*}
    Then, by $\chi^2-$concentration, $\proba{\chi_r^2\geq r/2}= 1-\proba{\chi_r^2-r< -r/2}\geq 1-e^{-c_1 r}$.
    Similarly, $\proba{\chi_1^2\leq r/2}\geq 1-e^{-c_2 r}$, leading to:
    \begin{align*}
        \proba{U{\rm\, is \,admissible}}&\geq (1-e^{-c_1 r})^n(1-e^{-c_2 r})^n\\
        &\geq \exp(-ne^{-c_3 r}))\,.
    \end{align*}
    Thus, for $n<C'e^{c'r}$, we have that this probability is greater than $7/8$.
\end{proof}

\paragraph{Proof of \Cref{prop:lower_m_1}}

\begin{proof}[Proof of \Cref{prop:lower_m_1}]
    Let $U$ be orthogonal to $U^\star$. $\NRM{U^\top x_1}^2$ and $y_1^2$ are independent (orthogonality) and respectively $\Xi_r^2$ and $\Xi_1^2$ random variables.
    We have that
    \begin{align*}
        \proba{U\text{ not admissible}}&=\proba{\exists i\in[n], \NRM{U^\top x_i}^2< y_i^2}\\
        &\leq n\proba{\NRM{U^\top x_1}^2< y_1^2}\\
        &\leq n \proba{\NRM{U^\top x_1}^2< r/2} + n\proba{y_1^2\geq r/2}\\
        &\leq n e^{-r/16} + ne^{-r/4}\,,
    \end{align*}
    concluding the proof.
\end{proof}
\subsection{Proof of \Cref{prop:lower_m_1}}

\begin{proof}[Proof of \Cref{prop:lower_m_1}]
    Let $U$ be orthogonal to $U^\star$. $\NRM{U^\top x_1}^2$ and $y_1^2$ are independent (orthogonality) and respectively $\Xi_r^2$ and $\Xi_1^2$ random variables.
    We have that
    \begin{align*}
        \proba{U\text{ not admissible}}&=\proba{\exists i\in[n], \NRM{U^\top x_i}^2< y_i^2}\\
        &\leq n\proba{\NRM{U^\top x_1}^2< y_1^2}\\
        &\leq n \proba{\NRM{U^\top x_1}^2< r/2} + n\proba{y_1^2\geq r/2}\\
        &\leq n e^{-r/16} + ne^{-r/4}\,,
    \end{align*}
    concluding the proof.
\end{proof}

\section{Proof of \Cref{thm:nuc_norm_main}}

\subsection{A general result}

\begin{theorem}\label{thm:nuc_norm_gen}
	Assume that each $f_i$ is $\mu$-strongly convex and that for some constants $\alpha,\beta>0$, for all $W,W',\Delta\in\R^{d\times n}$ of rank at most $s\geq 2r$, we have
	\begin{equation*}
		\left|\langle\nabla \cL(W^\star),\Delta\rangle \right|\leq \alpha \sigma_\star^2 \sup_{i\in[n]}\NRM{\Delta_i}_2^2 +\sigma_\star^2\sqrt{\frac{\alpha}{n}\NRM{\Delta}_F^2\sup_{i\in[n]}\NRM{\Delta_i}_2^2}\,,
	\end{equation*}
	and\begin{equation*}
		\left|D_{f-\cL}(W,W')\right|\leq \beta \sigma^2 \sup_{i\in[n]}\NRM{w_i-w_i'}_2^2 +\sigma^2\sqrt{\frac{\beta}{n}\NRM{W-W'}_F^2\sup_{i\in[n]}\NRM{w_i-w_i'}_2^2}\,.	
	\end{equation*}
	Assume furthermore that $\cL$ satisfies $D_\cL(W,W')\leq M\frac{\NRM{W-W'}_F^2}{n}$ for all $W,W'\in\R^{d\times n}$. Then, 
	We have:
	\begin{align*}
		\frac{1}{n}\NRM{\hat W_\SVD-W^\star}_F^2&\leq \frac{2M\kappa^2 r}{\mu s} + \frac{4B}{\mu}\left[\sigma_\star^2\left(\frac{\alpha}{s} + \sqrt{\frac{\alpha}{ns}}\right)\right]\kappa\sqrt{nr} \\
		&\quad + \frac{16B^2\beta \sigma^2}{\mu} + \frac{8B^2\sigma_\star^2}{\mu}\left(\alpha   +\sqrt{\alpha}\right) \,.
	\end{align*}
\end{theorem}

\begin{proof}
Importantly, notice that $W^\star$ lies in the optimization set
\begin{equation*}
    \cW=\set{  W\in\R^{d\times n}\,\big|\quad \sup_i \NRM{w_i}\leq B\,,\quad \NRM{W}_*\leq \kappa B\sqrt{nr}}\,.
\end{equation*}
Indeed, by assumption we first have that $\sup_i \NRM{w_i^\star}\leq B$, and $\NRM{W^\star}_*\leq r\lambda_{\max}(W)\leq \kappa\sqrt{rn}$.
The scaling of $\kappa$ then comes from the fact that if $W^\star$ is of rank $r$ and is well-conditioned (its largest and $r-$largest singular values $\lambda_1(W^\star)$ and $\lambda_r(W^\star)$ have a bounded ratio), 
we have \textit{(i)} $r \lambda_r(W^\star)^2\NRM{W^\star}_F^2=\sum_i\NRM{w_i^\star}_2^2\leq nB^2$ leading to $\lambda$ and \textit{(ii)} $\NRM{W^\star}_*\leq r\lambda_1(W^\star)$, so that if $\kappa\geq \frac{\lambda_1(W^\star)}{\lambda_r(W^\star)}$, we have $\NRM{W^\star}_*\leq r\lambda_1(W^\star)\leq \kappa r \lambda_r(W^\star)\leq \kappa B \sqrt{nr}$.
We have:
	\begin{equation*}
		D_{\cL}(\hat W,W^\star)\leq \langle \nabla \hat\cL(W^\star),W^\star-\hat W\rangle\,.
	\end{equation*}
	We first bound the noise term $\langle \nabla \hat\cL(W^\star),W^\star-\hat W\rangle$.
	Let $\Delta= W^\star-\hat W$ and $\Delta=\sum_{k=1}^{K}$ for $K\leq \lceil \frac{\min(n,d)}{s}\rceil$ be its $s$-shelling decomposition, defined as follows.

	\begin{definition}\label{def:shelling}
		Let $W\in\R^{p\times q}$ and $s\in\N^*$.
		Let $W=\sum_{k=1}^{\min(p,q)} \sigma_k u_vv_k^\top$ be the singular value decomposition of $W$, where $\sigma_1\geq\sigma_2\geq\ldots\geq0$.
		The top-$s$-SVD of $W$, abbreviated as $\SVD_s(W)$ is defined as $\SVD_s(W)=\sum_{k=1}^s \sigma_k u_v v_k^\top$.
	\end{definition}
	
	\begin{definition}
		Let $W\in\R^{p\times q}$ and $s\in\N^*$.
		Let $W=\sum_{k\geq 1} \Delta^{(k)}$ be the ‘‘shelling decomposition'' of $W$, where the sequence of matrices $(\Delta^{(k)})_{k\geq1}$ is defined as:
		\begin{enumerate}
			\item $\Delta^{(1)}=\SVD_s(W)$;
			\item Recursively, for $k\geq2$, $\Delta^{(k)}=\SVD_s(W-\sum_{\ell<k}\Delta^{(\ell)})$.
		\end{enumerate}
		Note that $\Delta^{(k)}=0$ for $k\geq \lceil \min(p,q)/s\rceil$.
	\end{definition}
	
	\begin{lemma}\label{lemma:shelling}
		For any $W\in\R^{p\times q}$ and $s\in\N^*$, writing $W=\sum_{k\geq 1}\Delta^{(k)}$ the shelling decomposition of $W$, we have:
	
		\begin{equation*}
			\sum_{k\geq 2 } \sup_{i\in[q]} \NRM{\Delta^{(k)}_i}_2 \leq \frac{\NRM{W}_*}{s} \,,
		\end{equation*}
		and	
		\begin{equation*}
			\sum_{k\geq 2} \NRM{\Delta^{(k)}}_F\leq \frac{\NRM{W}_*}{\sqrt{s}}\,.
		\end{equation*}
	\end{lemma}

	We thus have:
	\begin{align*}
		\langle \nabla \hat\cL(W^\star),W^\star-\hat W\rangle&\leq \sum_{k\geq 1}\langle \nabla \hat\cL(W^\star),\Delta^k\rangle\\
		&\leq \sum_{k\geq 1}\alpha \sigma_\star^2 \sup_{i\in[n]}\NRM{\Delta_i}_2^2 +\sigma_\star^2\sqrt{\frac{\alpha}{n}\NRM{\Delta}_F^2\sup_{i\in[n]}\NRM{\Delta_i}_2^2}\\
		&\leq 4B^2\left(\alpha \sigma_\star^2  +\sigma_\star^2\sqrt{\alpha}\right)+\sum_{k\geq 2}\alpha \sigma_\star^2 \sup_{i\in[n]}\NRM{\Delta_i}_2^2 +\sigma_\star^2\sqrt{\frac{\alpha}{n}\NRM{\Delta}_F^2\sup_{i\in[n]}\NRM{\Delta_i}_2^2}\\
		&\leq 4B^2\sigma_\star^2\left(\alpha   +\sqrt{\alpha}\right) + \sigma_\star^2\NRM{\Delta}_*\left(\frac{\alpha}{s} + B\sqrt{\frac{\alpha}{ns}}\right)\\
		&\eqdef \eps_0\,. 
	\end{align*}
	Then, we have that $\NRM{\Delta}_*\leq 2\kappa\sqrt{nr}$,
	and thus 
	\begin{equation*}
		\eps_0\leq 4B^2\sigma_\star^2\left(\alpha   +\sqrt{\alpha}\right) + 2\kappa\sqrt{nr}\sigma_\star^2\left(\frac{\alpha}{s} + B\sqrt{\frac{\alpha}{ns}}\right)\,.
	\end{equation*}
	We now need to lower bound $D_{\cL}(\hat W,W^\star)$.
	Let $\hat W=\sum_{k=1}^K\hat\Delta^k$ be its $s$-shelling decomposition, and let $\tilde\Delta^k=\hat \Delta^k$ for $k\geq 2$ and $\tilde\Delta^1=\hat\Delta^1-W^\star$.
	Let $\tilde{\Delta}=\hat W-W^\star=\sum_k\tilde{\Delta}^k$.
	We have:
	\begin{align*}
		D_{\cL}(\hat W,W^\star)&=D_{\cL}( W^\star +\tilde{\Delta},W^\star)\\
		&=D_{\cL}( W^\star +\sum_{k=1}^K\tilde{\Delta}^k,W^\star)\\
		&=D_{\cL}( W^\star +\tilde{\Delta}^1+\sum_{k=2}^K\tilde{\Delta}^k,W^\star)\\
		&\geq D_{\cL}( W^\star +\tilde{\Delta}^1+\sum_{k=2}^K\tilde{\Delta}^k,W^\star)\\
		&\geq 2D_{\cL}( W^\star + \frac{1}{2}\tilde{\Delta}^1,W^\star)-D_{\cL}( W^\star - \sum_{k=2}^K\tilde{\Delta}^k,W^\star)\\
		&\geq 2D_{\cL}( W^\star + \frac{1}{2}\tilde{\Delta}^1,W^\star)-\frac{1}{K-1}\sum_{k=2}^KD_{\cL}( W^\star - (K-1)\tilde{\Delta}^k,W^\star)\,,
	\end{align*}
	where we used twice the convexity of $D_\cL(W^\star+\cdot,W^\star)$.

	Then, 
	\begin{align*}
		D_{\cL}( W^\star + \frac{1}{2}\tilde{\Delta}^1,W^\star)&=D_{f}( W^\star + \frac{1}{2}\tilde{\Delta}^1,W^\star)+ D_{\cL-f}( W^\star + \frac{1}{2}\tilde{\Delta}^1,W^\star)\\
		&\geq D_{f}( W^\star + \frac{1}{2}\tilde{\Delta}^1,W^\star) - \sigma^2\left(\beta \sup_i\NRM{\tilde{\Delta}^1_i}_2^2+\sqrt{\frac{\beta}{n}\NRM{\tilde{\Delta}^1}^2_F\sup_i\NRM{\tilde{\Delta}^1_i}_2^2}\right)\\
		&\geq D_{f}( W^\star + \frac{1}{2}\tilde{\Delta}^1,W^\star) - \sigma^2\left(4B^2\beta+ 2B\sqrt{\frac{\beta}{n}\NRM{\tilde{\Delta}^1}^2_F}\right)\,,
	\end{align*}
	and using the property with constant $M$:
	\begin{align*}
		D_{\cL}( W^\star - \sum_{k=2}^K\tilde{\Delta}^k,W^\star)& \leq \frac{M}{n}\NRM{\sum_{k=2}^K\tilde{\Delta}^k}_F^2\\
		&\leq  \frac{M\NRM{W^\star}_*^2}{ns}\\
		&\leq \frac{M\kappa^2 r}{s}\\
	\end{align*}
	Now, wrapping everything up,
	\begin{align*}
		2(f(W^\star +\tilde\Delta^1/2)-f(W^\star))&\leq\frac{M\kappa^2 r}{s}+ 2\left[\sigma_\star^2\left(\frac{\alpha}{s} + B\sqrt{\frac{\alpha}{ns}}\right)\right]\kappa\sqrt{nr} \\
		&\quad + \sigma^2\left(4B^2\beta+ 2B\sqrt{\frac{\beta}{n}\NRM{\tilde\Delta^1}^2_F}\right)  + 4B^2\sigma_\star^2\left(\alpha   +\sqrt{\alpha}\right) \,.
	\end{align*}
	Then, if each $f_i$ is $\mu$ strongly convex, $f$ is $\mu/n$ strongly convex and:
	\begin{align*}
		\frac{\mu}{n}\NRM{\tilde\Delta^1}_F^2&\leq \frac{M\kappa^2 r}{s}+2\left[\sigma_\star^2\left(\frac{\alpha}{s} + B\sqrt{\frac{\alpha}{ns}}\right)\right]\kappa\sqrt{nr} \\
		&\quad + \sigma^2\left(4B^2\beta+ 2B\sqrt{\frac{\beta}{n}\NRM{\tilde\Delta^1}^2_F}\right)  + 4B^2\sigma_\star^2\left(\alpha   +\sqrt{\alpha}\right) \,.
	\end{align*}
	Thus,
	\begin{align*}
		\frac{1}{n}\NRM{\hat W_\SVD-W^\star}_F^2&\leq \frac{2M\kappa^2 r}{\mu s} + \frac{4}{\mu}\left[\sigma_\star^2\left(\frac{\alpha}{s} + B\sqrt{\frac{\alpha}{ns}}\right)\right]\kappa\sqrt{nr} \\
		&\quad + \frac{16B^2\beta \sigma^2}{\mu} + \frac{8B^2\sigma_\star^2}{\mu}\left(\alpha   +\sqrt{\alpha}\right) \,.
	\end{align*}

\end{proof}

\subsection{Proof of \Cref{thm:nuc_norm_main}}

The first two assumptions of \Cref{thm:nuc_norm_gen} are satisfied with high probability for $\alpha,\beta=\tilde\cO(\frac{s(d+n)}{nm})$, as proved in the proof of \Cref{thm:low_rank}. We now prove that the third and last assumption is satisfied for $M=L+\tilde\cO(\sigma^2d/m)$ with high probabiliy.

\begin{proof}

\begin{lemma}
		With probability $1-e^{-cmd}$, for all $i\in[n]$, for all $w,w'\in\R^d$, 
		\begin{equation*}
			D_{\cL_i}(w,w')\leq \left(L + c'\sigma^2\frac{d}{m}\ln(dHB^2/\sigma^2) \right) \NRM{w-w'}^2
		\end{equation*}
	\end{lemma}
	\begin{proof}
		We have, for some $v\in[w,w']$:
		\begin{align*}
			D_{\cL_i}(w,w')&=\frac{1}{2}\NRM{w-w'}^2_{\nabla^2\cL_i(v)}\\
			&=\frac{1}{m}\sum_{j} (w-w')^\top 
			\nabla^2F_i(v,\xi_{ij}) (w-w')\,.	
		\end{align*}
		Assume that $\NRM{w-w'}=1$.
		Now, all $(w-w')^\top 
		\nabla^2F_i(v,\xi_{ij}) (w-w')$ are $\sigma^2-$subexponential and independent, of mean $(w-w')^\top 
		\nabla^2f_i(v) (w-w')\leq L$, leading to:
		\begin{align*}
			\proba{ \sum_{j} (w-w')^\top 
			\nabla^2F_i(v,\xi_{ij}) (w-w') \geq mL + \sigma^2t }\leq 2\exp\left(-c\min(\frac{t^2}{m},t)\right)\,,
		\end{align*}
		so that:
		\begin{align*}
			\proba{\frac{1}{m} \sum_{j} (w-w')^\top 
			\nabla^2F_i(v,\xi_{ij}) (w-w') \geq L + \sigma^2t }\leq 2\exp\left(-cm\min(t^2,t)\right)\,.
		\end{align*}
		Using discretization arguments as before, both on $v\in[w,w']$ parameterized by some $\lambda\in[0,1]$ and on $w,w'$, we get that, with probability
		$1-e^{-cmd}$, for all $w,w'$ such that $\NRM{w},\NRM{w'}\leq B$,
		\begin{equation*}
			D_{\cL_i}(w,w')\leq \left(L + c'\sigma^2\frac{d}{m}\ln(dHB^2/\sigma^2) \right) \NRM{w-w'}^2
		\end{equation*}

\end{proof}

	\end{proof}

\end{document}